%% file: main.tex
\documentclass[journal]{IEEEtai}

\usepackage[colorlinks,urlcolor=blue,linkcolor=blue,citecolor=blue]{hyperref}

\usepackage{color,array}

\usepackage{graphicx}

\input{macros}

\begin{document}

\title{Replication Robust Payoff Allocation in Submodular Cooperative Games}

\author{Dongge Han, Michael Wooldridge, Alex Rogers, Olga Ohrimenko, Sebastian Tschiatschek
	\thanks{Manuscript received Jan 15, 2022.}
	\thanks{Dongge Han, Michael Wooldridge and Alex Rogers are with the University of Oxford, OX1 3QD, Oxford, U.K.  (e-mail: dongge.han@cs.ox.ac.uk, michael.wooldridge@cs.ox.ac.uk, alex.rogers@cs.ox.ac.uk).}
	\thanks{Olga Ohrimenko is with the University of Melbourne, Victoria 3010, Australia (e-mail: oohrimenko@unimelb.edu.au).}
	\thanks{Sebastian
Tschiatschek is with University of Vienna, Währinger Straße 29 1090 Wien, Austria (e-mail: sebastian.tschiatschek@univie.ac.at).}}

\markboth{Journal of IEEE Transactions on Artificial Intelligence, Vol. 00, No. 0, Month 2022}
{Dongge Han \MakeLowercase{\textit{et al.}}: Replication Robust Payoff Allocation in Submodular Cooperative Games}

\maketitle

\begin{abstract}
Submodular functions have been a powerful mathematical model for a wide range of real-world applications. Recently, submodular functions are becoming increasingly important in machine learning (ML) for modelling notions such as information and redundancy among entities such as data and features. Among these applications, a key question is payoff allocation, i.e., how to evaluate the importance of each entity towards the collective objective? To this end, classic solution concepts from cooperative game theory offer principled approaches to payoff allocation. However, despite the extensive body of game-theoretic literature, payoff allocation in submodular games are relatively under-researched. In particular, an important notion that arises in the emerging submodular applications is redundancy, which may occur from various sources such as abundant data or malicious manipulations where a player replicates its resource and act under multiple identities. Though many game-theoretic solution concepts can be directly used in submodular games, naively applying them for payoff allocation in these settings may incur robustness issues against replication.
In this paper, we systematically study the replication manipulation in submodular games and investigate \emph{replication robustness}, a metric that quantitatively measures the robustness of solution concepts against replication. Using this metric, we present conditions which theoretically characterise the robustness of semivalues, a wide family of solution concepts including the Shapley and Banzhaf value. Moreover, we empirically validate our theoretical results on an emerging submodular ML application, i.e., the ML data market.
\end{abstract}

\begin{IEEEImpStatement}
With the increasing take-up of ML techniques in real-world settings, payoff allocation has significant impacts towards fairness, trustworthiness, safety, and knowledge discovery in ML applications, e.g., performing analysis or debugging of ML systems by finding the key contributors or bottleneck entities. Many emerging ML applications exhibit submodular characteristics, while properties of classic game-theoretic payoff allocation on submodular games are under-researched. This paper investigated an important issue of redundancy arising from replication in the submodular ML applications. Using the replication robustness metric, we provide theoretical guarantees for the robustness of common game-theoretic payoff allocation methods against replication. Our findings can guide the use of game-theoretic payoff allocation in submodular ML applications, and impact real-world applications and future research on payoff allocation in ML systems in general, such as fair compensation in multi-party ML systems and feature importance interpretation in the medical domains.
\end{IEEEImpStatement}


\section{Introduction}\label{market:intro}
\input{sections/intro}

\section{Background}
\input{sections/background}
\section{Submodular Games and Replication}\label{market:theory}
\input{sections/method}

\section{Case Study: ML Data Markets}\label{market:application}
\input{sections/market}

\subsection{Experiments}\label{market:experiments}
\input{sections/experiments}

\section{Related Work}\label{market:related_work}
\input{sections/related}

\section{Conclusions}\label{market:conclusions}
\input{sections/conclusions}



\bibliographystyle{abbrvnat}
\bibliography{sample}

\clearpage

\onecolumn
\appendices
\input{sections/appendix}

\end{document}

%% file: macros.tex
\usepackage{tikz}
\usetikzlibrary{fit}
\usetikzlibrary{arrows.meta}
\usepackage{amsmath,amsthm, amsfonts, mathtools,nccmath}
\usepackage{graphicx}
\usepackage{subfig}
\usepackage{pdfpages}
\usepackage{subfiles}
\usepackage{float}
\usepackage[square,numbers]{natbib}
\usepackage{multirow}
\usepackage{graphicx}
\usepackage{xcolor}
\usepackage{algorithm}
\usepackage{algpseudocode}
\usepackage{wrapfig}
\usepackage{aligned-overset}
\usepackage{thm-restate}
\let\labelindent\relax
\usepackage{enumitem}
\usepackage{booktabs}
\usepackage{pifont}

\newtheorem{theorem}{Theorem}[subsection]
\newtheorem{definition}{Definition}[subsection]
\newtheorem{assumption}{Assumption}

\newtheorem{corollary}{Corollary}[subsection]
\newtheorem{lemma}{Lemma}[subsection]

\newtheorem{example}{Example}
\newcommand{\loc}{\ensuremath{\mathcal{L}}}
\newcommand{\val}{\ensuremath{\varphi}} 
\newcommand{\vshapley}{\ensuremath{\varphi^{\textnormal{Shapley}}}}
\newcommand{\vbanzhaf}{\ensuremath{\varphi^{\textnormal{Banzhaf}}}}
\newcommand{\vloo}{\ensuremath{\varphi^{\textnormal{LOO}}}}
\newcommand{\validation}{\ensuremath{\textnormal{val}}} 

\newcommand{\mc}{\ensuremath{\mathit{MC}}}
\newcommand{\weight}[2]{\ensuremath{w_{#1,#2}}}
\newcommand{\cv}{\ensuremath{v}}
\newcommand{\coalition}{\ensuremath{\mathcal{C}}}  

\renewcommand{\S}{\ensuremath{\mathcal{C}}}
\newcommand{\model}{\ensuremath{\mathcal{M}}}
\usepackage{float}
\usepackage{bbm}
\usepackage[export]{adjustbox}

%% file: sections/intro.tex

Submodularity has long been an important topic in mathematics, operations research, economics and optimisation. Submodular functions~\cite{schrijver2003combinatorial} exhibit the natural property of \emph{diminishing returns}. Informally, given a ground set of elements (e.g., physical entities such as sensors, goods, or digital entities such as data, features), the marginal contribution of a single element when added to a set of elements diminishes with the increasing size of the set. This property frequently occurs in real-world settings, making submodular functions a powerful mathematical model for a wide range of applications, such as cooperative cost allocations~\cite{goemans2004cooperative}, sensor placement~\cite{krause2008near} and facility location problems (FLP)~\cite{cornuejols1977uncapacitated}. Recently in the field of machine learning (ML), submodular functions are becoming increasingly important as they naturally model notions of \emph{information, diversity and redundancy}~\cite{bilmes2017deep}. In these classic and emerging applications, a key question is \emph{how to evaluate the importance of each entity towards the collective objective, i.e., payoff allocation?} On the one hand, in cooperative settings, importance evaluations can enable fair allocation of the collective reward towards each member. On the other hand, evaluating the importance of each entity helps to identify crucial insights into the system such as the key contributors or redundant entities. An example use case is ML model interpretation~\cite{ribeiro2016should, lundberg2017unified} -- typically a trained blackbox ML model cannot be interpreted by humans. To interpret the model and ensure it is trustworthy, we can evaluate the importance it gives to each input feature when making a prediction. 

A principled approach to payoff allocation is provided by cooperative game theory~\cite{chalkiadakis2011computational}, which models the entities as players and their interactions (typically) in the form of a characteristic function game $G=(N, \cv)$, where a characteristic function $\cv$ evaluates each possible set of players. Under this formulation, the most popular game-theoretic solution concept is the Shapley value~\cite{Shapley+2016+307+318}, which allocates the payoff to each player as a weighted average of its' marginal contributions towards all possible sets of other players, 
and has been widely applied in network centrality~\cite{aadithya2010efficient}, ML interpretation~\cite{lundberg2017unified}, data valuation~\cite{jia2019towards,agarwal2019marketplace}, etc. 
Despite the extensive body of game-theoretic literature, 
submodular games (i.e., games with submodular characteristic functions) are relatively under-explored, a setting where players may not be incentivised to cooperate and form a grand coalition. Nevertheless, with the ever-grown interest in ML applications, the above setting becomes increasingly common than ever and lead to an urgent need to study payoff allocation in submodular games. In fact, many problems in ML are submodular by nature, and players form a grand coalition inherently (e.g., among passive entities such as data and features) or according to rules which require the cooperation among players. For example, consider multiple hospitals collaboratively training an ML model by pooling their medical images, the hospitals will agree to cooperate and form the grand coalition in order to train a better prediction model, even though a player may be less useful in terms of marginal contributions with increasing data.

Closely related to the submodular games is the notion of \emph{redundancy}~\cite{bilmes2017deep}. On the one hand, redundancy may come from a benign source, e.g., abundant data typically carry partially redundant information and yields diminishing returns. This motivates important problems such as data selection~\cite{wei2015submodularity, kirchhoff2014submodularity}, feature selection~\cite{das2012selecting} and data summarisation~\cite{lin2011class}. On the other hand, redundancy may arise as a result of malicious manipulations, e.g., replication manipulation, where a malicious player may replicate its resource and act under multiple false identities. In both the malicious and benign cases, redundancy often does not bring significant additional value to the collective objective, but may have substantial impact on the payoff allocation. 
Though many common game-theoretic solution concepts can be directly applied to the emerging submodular ML applications, there is no theoretical guarantees for these solution concepts in terms of redundancy. Consequently, naively applying them for payoff allocation in these settings may incur robustness issues such as incentivizing the aforementioned replication manipulation. 

In this paper, we systematically study the replication manipulation in submodular games and investigate \emph{replication robustness}, a metric which quantitatively measures the robustness of solution concepts against replication manipulations. Using this metric, we present conditions which theoretically characterise the robustness of semivalues~\cite{dubey1981value}, a wide family of Shapley-like solution concepts including the Shapley value and the Banzhaf value~\cite{lehrer1988axiomatization}. Though we model the redundancy from the perspective of malicious manipulations, the theoretical framework can also be extended to study redundancy that occur under the benign cases, for example, for promoting diversity among features in ML feature subset selections, or encourage diverse behaviours among robotic agents in multiagent reinforcement learning.

The outline of our paper is as follows:
In Section~\ref{market:theory} we first define submodular games, the replication manipulation and replication robustness. To illustrate the effect of redundancy, we look at a classic submodular problem -- the facility location problem. In Section~\ref{sec:k=1}, we compare the replication robustness of the Shapley value and the Banzhaf value when a malicious player replicates its resource and acts as two identities. In Section~\ref{sec:k>1}, we extend our theoretical results to general semivalues and an arbitrary number of replications, and we present a necessary and sufficient condition which characterises the replication robustness of general semivalues. 
Finally in Section~\ref{market:application}, we apply our theoretical results to an emerging ML application -- the ML data market~\cite{ohrimenko2019collaborative, agarwal2019marketplace}, and empirically validate our theoretical results of replication robustness across various solution concepts.

%% file: sections/background.tex
In this section, we introduce our notation and concepts from cooperative game theory~\cite{chalkiadakis2011computational}.

\textbf{Cooperative Games.}
Formally, a cooperative game with transferable utility (hereafter simply a \emph{cooperative game}) is given by a tuple $G = (N, v)$, where $N = \{1,\ldots,n\}$ is the set of players of the game
and $v\colon 2^N\rightarrow{\mathbb{R}}$ is a \emph{characteristic function}, which assigns a real value 
$v(\S)$ to every subset of players $\S \subseteq N$, referred to as \emph{coalitions}. The \emph{grand coalition} is the set N of all players.  
For clarity, we will introduce the general definition of semivalues~\cite{dubey1981value} in Section~\ref{subsec:zc}. Before this, we introduce here the concept of marginal contribution and some common semivalues. Intuitively, the \emph{marginal contribution} of a player to coalition $\S$ is
the difference that this player makes towards $\S$ before and after
joining it, i.e., $\mc_i(\S) \coloneqq v(\S\cup\{i\}) - v(\S)\nonumber$.

\textbf{Solution Concepts.}
A solution concept~\cite{chalkiadakis2011computational} describes the outcome of a cooperative game, i.e., the partition of players into coalitions, and a payoff function which assigns a payoff $\val_i(N,v) \in \mathbb{R}$ to each player $i$. 
As discussed in the introduction, we focus on payoff allocations in the emerging ML settings where the players form the grand coalition inherently. Therefore, we will refer to the solution concepts as the payoff allocation with respect to the grand coalition. 

The following is a collection of properties which are commonly used to axiomatize solution concepts~\cite{chalkiadakis2011computational}.
\begin{enumerate}[label=\textbf{(A\arabic*)}]
    \itemsep 0mm
	\item \label{ax:symmetry} \emph{Symmetry}: Two players $i$ and $j$ who have the same marginal contribution in any coalition have the same payoff, i.e.,
	 $(\forall \S \subseteq N\setminus{\{i,j\}}\colon v(\S\cup\{i\}) = v(\S\cup\{j\})) \rightarrow \val_i(N,v) = \val_j(N,v)$.
	\item\label{ax:efficiency} \emph{Efficiency}: The payoff values of all players sum to $v(N)$, i.e., $v(N) = \sum_{i \in N} \val_i(N,v)$.
	\item \label{ax:nullplayer} \emph{Null-player}: a player whose marginal contribution is zero in any coalition has zero payoff, i.e., 
	$(\forall \S \subseteq N\colon v(\S\cup\{i\}) = v(\S)) \rightarrow \val_i(N,v) = 0$.
	\item \label{ax:linearity} \emph{Linearity}: Given two cooperative games $G^1= (N, v^1)$ and $G^2 = (N, v^2)$, then for any player $i\in N$, $\varphi_i(N,v^1+v^2) = \varphi_i(N,v^1) + \varphi_i(N,v^2)$.
	\item \label{ax:twoeff} \emph{2-Efficiency}~\cite{lehrer1988axiomatization}:
    $\val_i(N,v) + \val_j(N,v) = \val_{p_{ij}}(N', v')$ characterises neutrality of collusion,
    where $\val_{p_{ij}}(N',v')$ is player $p_{ij}$'s payoff in a game in which players $i$ and $j$ merged as a single player $p_{ij}$, i.e., $N' = N\setminus\{i,j\}\cup\{p_{ij}\}$. 
\end{enumerate}

Next, we review three common semivalues.
\begin{itemize}[leftmargin=*]
  \item The {\it Shapley Value}~\cite{Shapley+2016+307+318} is the most common solution concept, defined as the weighted average marginal contributions of a player towards coalitions of other players, and the unique value that satisfies \ref{ax:symmetry}-\ref{ax:linearity}:
\begin{equation*}
\vshapley_i = \sum_{S\subseteq N \setminus \{i\}} \frac{|\S|!(|N|-|\S| -1)!}{|N|!} \mc_i(\S) 
\end{equation*}

  \item The {\it Banzhaf Value}~\cite{lehrer1988axiomatization} is commonly used as a measure for voting power, which is defined by the average marginal contribution of a player towards all coalitions of other players, uniquely characterized by axioms \ref{ax:symmetry}, \ref{ax:nullplayer}-\ref{ax:twoeff}:
\begin{equation*}
  \vbanzhaf_i = \frac{1}{2^{|N|-1}}\sum_{\S\subseteq N \setminus \{i\}} \mc_i(\S)
\end{equation*}

\item{\it Leave-one-out (LOO)}
assigns to each player its marginal contribution towards the coalition of all other players:\begin{equation*}
  \vloo_i = \mc_i(N\setminus\{i\}).
 \end{equation*}
\end{itemize}

%% file: sections/method.tex
We now introduce submodular functions and how they can be used as the characteristic functions in cooperative games. To illustrate, we show an example class of submodular games defined by a classic submodular function -- the facility location function. We will also use this example to validate our theoretical findings in Section~\ref{subsec:facility_location_robustness}. Following the definition of submodular games, we will show how replication manipulations can be performed, and define the criteria which evaluates the robustness of solution concepts against replication. 

\subsection{Submodular Games}
The following property lists three equivalent definitions of submodular set functions (aka submodular functions)~\cite{schrijver2003combinatorial}:
\begin{definition}[Submodular Set Functions]\label{def:submodularity}
  Let $N$ be a finite set, a submodular function is a set function $f:2^{N}\rightarrow \mathbb{R}$, where $2^N$ denotes the power set of $N$ , which satisfies one of the following equivalent conditions:
  \begin{itemize}
      \item $\forall X, Y \subseteq N$ with $X \subseteq Y$ and $\forall x\in N \setminus Y$, we have $ f(X\cup \{x\})-f(X)\geq f(Y\cup \{x\})-f(Y)$.
      \item $\forall S, T \subseteq N$, we have $f(S)+f(T)\geq f(S\cup T)+f(S\cap T)$.
      \item $\forall X\subseteq N$ and $x_{1},x_{2}\in N\backslash X$ such that $x_{1}\neq x_{2}$, we have $f(X\cup \{x_{1}\})+f(X\cup \{x_{2}\})\geq f(X\cup \{x_{1},x_{2}\})+f(X)$.
  \end{itemize}
\end{definition}

The first one of the equivalent conditions demonstrates diminishing returns, i.e., the marginal value of an entity towards a set decreases as the set grows. Due to its natural relation to the marginal contributions of cooperative games, we next define submodular games using the first condition. 

\begin{definition}[Submodular Game]
	A characteristic function game $G = (N,v)$ with a finite non-empty set of players $N = \{1, \ldots, n\}$, is a \emph{submodular game} if the characteristic function $\cv$ is submodular, i.e, $\forall \S \subseteq \S' \subseteq N \setminus \{i\}\colon \cv(\S \cup \{i\}) - \cv(\S) \geq \cv(\S' \cup \{i\}) - \cv(\S'). $
\end{definition}

Recall that the difference in value made by a player $i$ by joining a coalition $\coalition$ is denoted as the \emph{marginal contribution} of player $i$ towards coalition $\coalition$, i.e., $\mc_i(\coalition) \coloneqq \cv(\coalition\cup\{i\}) - \cv(\coalition)$. Therefore in a submodular game, the marginal contribution of a player towards a coalition $\coalition$ is no less than its contribution towards a superset $\coalition'$, as summarised in the next assumption.

\begin{assumption}
  \label{assumption:submodularity}
  In a submodular game $G = (N,v)$, the marginal contributions of each player $i\in N$ satisfy
  \begin{equation}
    \forall \S \subseteq \S' \subseteq N \setminus \{i\}\colon \mc_i(\S) \geq \mc_i(\S').\label{eq:submodular_mc}
  \end{equation}
\end{assumption}


\subsection{Motivating Example: Facility Location Problem}\label{app_sub:facility_location}

A classic example in submodular optimisation is the facility location problem (FLP). As an important topic in operations research, an FLP considers the question of how to select a cost-effective subset from a ground set of potential locations for placing new facilities~\cite{bilmes2017deep, Salhi1991DiscreteLT, cornuejols1977uncapacitated, fisher1978analysis}. Here, the facilities can refer to hospitals, plants, docking stations, etc. There exist several different formulations of the FLP~\cite{nemhauser1978analysis, cornuejols1977uncapacitated}, and we will adopt the formulation following~\citet{nemhauser1978analysis}, which consists a set of potential facility sites $\loc$ where new facilities can be opened, a set of customers $D$ to be serviced, and a matrix $U$ which represents utilities of each customer from each facility location (e.g., proximity). The FLP is un-capacitated, i.e, it is always optimal to satisfy the demand of a customer from the open facility which provides them with the highest utility. 
By modelling the FLP as a submodular game, we can evaluate the importance of each facility location by computing their payoff allocations using the common solution concepts. To do this, we can consider the players as the set of facility locations $\loc$, and the characteristic function as the facility location function $Fac(\S)=\sum_{d\in D} \max_{i \in \S} u_{id}$, i.e., the value of each coalition $\S\subseteq \loc$ is the sum of utilities of all customers from the open facilities $i\in\S$.

 \begin{example}[Facility Location Game]\label{example:facility_location_game}
     Let $D$ be a set of customers and $\loc$ a set of facility locations. Define a utility function $u: \loc\times D \rightarrow \mathbb{R}_+$, represented by a matrix $U \in \mathbb{R}_+^{|\loc| \times d}$, where each entry $u_{id} \in U$ is the utility of customer $d$ for facility location $i$.
     A facility location game is defined as $G = (\loc,\cv)$, where the players $\loc$ are facility locations and the characteristic function is the facility location function, i.e., $\forall \S\subseteq\loc, \cv(\S) = Fac(\S)=\sum_{d\in D} \max_{i \in \S} u_{id}$.
\end{example}

\begin{figure}[t]
	\centering
	\includegraphics[width=0.8\linewidth]{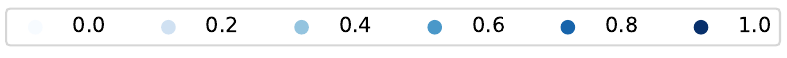}
	\vfill
	\subfloat[The Shapley Value]{\includegraphics[width=0.5\linewidth]{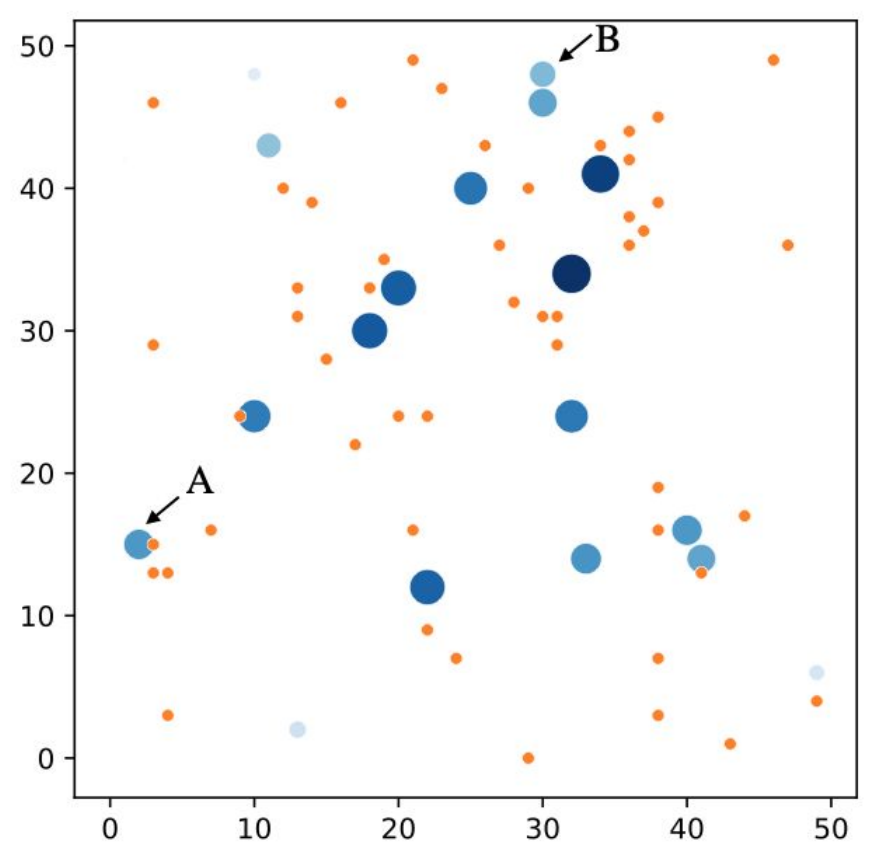}\label{subfig:facility_shapley}}%
	\hfill
	\subfloat[The Banzhaf value]{\includegraphics[width=0.5\linewidth]{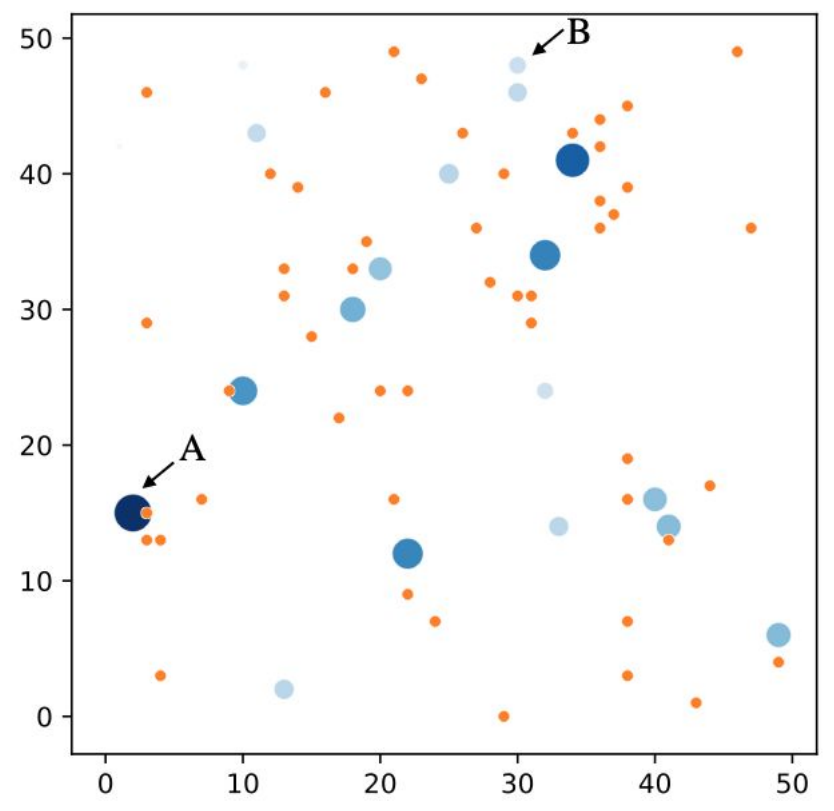}\label{subfig:facility_banzhaf}}%
	\caption{The Shapley value and Banzhaf value for the Facility Location Game. The figures show a 50x50 map, where orange dots are 50 customers. Blue dots refer to 20 facility locations, with larger and darker dots as larger Shapley/Banzhaf values, which are normalised between [0,1] for a clear comparison. }
	\label{fig:facility_location_results}
\end{figure}

Fig.~\ref{fig:facility_location_results} illustrates the Shapley and Banzhaf value on an example facility location game with $|\loc|=20$ facility locations (blue), and $|D| =50$ customers (orange) randomly placed in a $50\times 50 $ map. The utility of a customer $u_{id} = 100 - (|x_i - x_d| + |y_i-y_d|)$ decreases with the Manhattan distance to the facility.
In comparison, the locations with higher Shapley value typically has a larger number of nearby customers, while a location with a higher Banzhaf value often has a larger number of nearby customers and fewer nearby facility locations, e.g., A is distant from nearby facilities, and ranks higher in terms of the Banzhaf value than the Shapley value, and conversely for B, which has multiple nearby facilities. This example provides an intuitive comparison between the Shapley value and the Banzhaf value against redundancy. We will further investigate the cause of their distinct behaviours in the rest of the paper.


\subsection{Replication Manipulation}\label{sec:solution_concepts_replicate}

As illustrated in the facility game in Fig.~\ref{fig:facility_location_results}, an important notion that commonly arises in submodular settings is \emph{redundancy}, which may occur naturally from abundant resources or from malicious manipulations such as replication. For example, a standard submodular ML problem is data summarisation~\cite{bilmes2017deep, lin2011class}, which aims to find a concise subset to represent the ground set of data, which reduces the redundancy among the data while maintaining the level of diversity. 

In the following definition, we introduce the replication manipulation, where a malicious player replicates its resource (e.g., digital entities such as online identities, data, features) and acts under multiple false identities. Here we model redundancy from the point of view of malicious manipulations, nevertheless, the theoretical framework can also be extended to study redundancy that occur under the benign cases.

\begin{definition}[Replication Manipulation]\label{def:manipulation}
	In a submodular game $G = (N,v)$, a (malicious) player $i$ executes a replication action $k$ times on its resources $D_i$ and acts as $k+1$ players $\S^R =\{i_0, i_1, \ldots, i_{k}\}$ each holding one replica of $D_i$. Denote the induced game as $G^R = (N^R, v^R)$, where the induced set of players are $N^R = N\setminus\{i\} \cup \S^R$, and the induced characteristic function $v^R$ satisfies $\forall \S \subseteq N\setminus \{i\}, \forall i_k\in \S^R\colon v^R(i_k\cup \S)=v(i\cup \S)$ and $v^R(\S)=v(\S)$. By replicating, player $i$ receives a \emph{total payoff} which is the sum of the payoff of all its $k+1$ replicas, i.e.,
	$
	\val_i^{\textnormal{tot}}(k) = \sum_{\kappa=0}^k \val_{i_\kappa}(N^R, v^R)
	$.
\end{definition}

The next assumption captures the fact that adding redundant resources to a coalition typically do not change the value of the coalition (e.g., redundant feature or replicated data). We refer to this property as \emph{replication redundancy} and formalize it in the following assumption: 
\begin{assumption}[Replication Redundancy]
	\label{assumption:replication}
	A replica does not contribute additional value to coalitions which already contain another replica or the original resource:
	\begin{equation*}
	\forall i, j \in \S^R\colon (i \in \S) \rightarrow \mc_j(\S) = 0.
	\end{equation*}
\end{assumption}

Despite the fact that redundant resources do not bring significant additional value to the collective objective, it may have substantial impact on the payoff allocation. For example, a malicious player may be able to gain a higher total payoff by performing the replication manipulation described in Definition~\ref{def:manipulation}.  
The next definition formalizes the notion of replication robustness of solution concepts, i.e., a property that ensures that a player through replication gains a total payoff no more than its original payoff.

\begin{definition}[Replication Robustness]\label{def:replication_robustness} 
	A solution concept $\varphi$ 
	is \emph{replication robust} if the payoff of the replicating player $i$ in the original game $G$ is no less than the total payoff of the player's replicas $\S^R$ in the induced game $G^R$ after replication, i.e.,
	\begin{equation*}
	\val_i(N,v) \geq \sum_{i_\kappa\in \S^R} \val_{i_\kappa}(N^R,v^R).
	\end{equation*}
\end{definition}

To illustrate the condition, consider a malicious player who aims to increase its payoff by performing replication manipulation, a solution concept that is replication robust can then be used to counteract such malicious behaviours. To see an example in the benign case such as feature importance interpretation, adding to a set of features $\coalition$ a feature $f'$ that is redundant to feature $f\in \coalition$ in the set can be considered as a replication manipulation, and a replication robust solution concept will allocate the two redundant features a total value no greater than the value of the feature $f$ on its own.

Having defined the replication manipulation and robustness criteria, we next study the behaviours of the semivalues under replication and their robustness properties.

\section{Replication Robustness of Common Semivalues with $k=1$ Replications}\label{sec:k=1}
To start with, we first take a look at the two most common semivalues, the Shapley value and the Banzhaf value, and study how the total payoff the malicious player changes if the player replicates its resource and splits into two identities.

\subsection{Robustness of the Shapley Value}\label{subsec:k=1_shapley}
The following theorem shows that the Shapley value is not replication robust in submodular games. Specifically, under payoff allocation according to the Shapley value, the malicious player can always obtain a non-negative gain in total payoff by replicating its resource and splitting into two identities. 
\begin{restatable}{theorem}{shapleysinglerobustness}\label{thm:shapleysinglerobustness}
	Let G= (N, \cv) be a submodular game with replication redundant characteristic function $\cv$, a player $i\in N$ replicates and obtains the total payoff as two identities $\S^R = \{i_1, i_2\}$ in the new game $G^R= (N^R, \cv^R)$.  
	By replicating, the changes in total payoff of player $i$ is:
	\begin{equation*}
	\delta \vshapley_i = \sum_{\S\subseteq N \setminus \{i\}} \frac{|\S|!(|N|-|\S| -1)!}{(|N|+1)!} (|N|-2|\S|-1)\mc_i(\S).
	\end{equation*}
	Moreover, the total payoff of player $i$ after replication is no less than its payoff in the original game, i.e, $\delta \vshapley_i  \geq 0.$
\end{restatable}

\begin{proof}
The derivations for the changes in total payoff is included in Appendix~\ref{appendix:shapleysinglepayoff}. Here we focus on showing that the value is non-negative, i.e., $\delta \vshapley_i\geq 0$. To prove this, we make use of the submodularity property, which compares the marginal contributions of player $i$ towards pairs of coalitions of other players $\S_1 \subseteq \S_2 \subseteq N\setminus\{i\}$. To pair the coalitions, we make two observations on $\delta \vshapley_i$: Given two coalitions  $\S_1,\S_2\subseteq N\setminus\{i\}$ with complementary sizes, i.e., $|\S_1| + |\S_2| =|N\setminus\{i\}|= |N|-1$, 
(1) their weights in $\delta \vshapley_i$ are opposite and adds up to zero,
(2) There are equal number of size $c$ and $|N|-1-c$ coalitions, i.e., $\tbinom{|N|-1}{c} = \tbinom{|N|-1}{|N|-1-c}$.
These suggest that we may find a bijective mapping between the size $c$ coalitions and their size $|N|-1-c$ supersets. 
Formally,  for any coalition size $c < (|N|-1)/2$, we look for a bijective mapping $f$ between coalitions with inclusion relations and of complementary sizes, that is, $f:\{\S_1\subseteq N\setminus{\{i\}}\mid |\S_1|=c\} \mapsto \{\S_2\subseteq N\setminus{\{i\}}\mid |\S_2|=|N|-1-c\}$ such that $\S_1\subseteq f(\S_1)$. The corner case where $c = (|N|-1)/2$ can be omitted as they have zero weight in $\delta\vshapley_i$, i.e., $\frac{c!(|N|-c -1)!}{(|N|+1)!} (|N|-2c-1) = 0.$
To show the existence of the bijective mapping, we model the coalitions and their inclusion relations ($\subseteq \textnormal{and} \supseteq$) by a bipartite graph (An example is shown in Figure~\ref{fig:halls}). For any coalition size $c < (|N|-1)/2$,
define bipartite graph $B_c = (L, R, E)$ where each vertex corresponds to a coalition, i.e., vertices $L = \{\S_1\subseteq N\setminus\{i\} \mid |\S_1| = c\}$ are the size $c$ coalitions, and vertices $R = \{\S_2\subseteq N\setminus\{i\} \mid |\S_2| = |N| -1 - c\}$ are the size $|N|-1-c$ coalitions, and $|L| = |R|$ from observation (2). Denote edges $E$ as the set inclusion relations, that is, $E = \{\{\S_1, \S_2\}\mid \S_1\in L, \S_2\in R, \S_1\subseteq \S_2\}$. 
\begin{figure}
    \centering
    \includegraphics[width=0.9\linewidth]{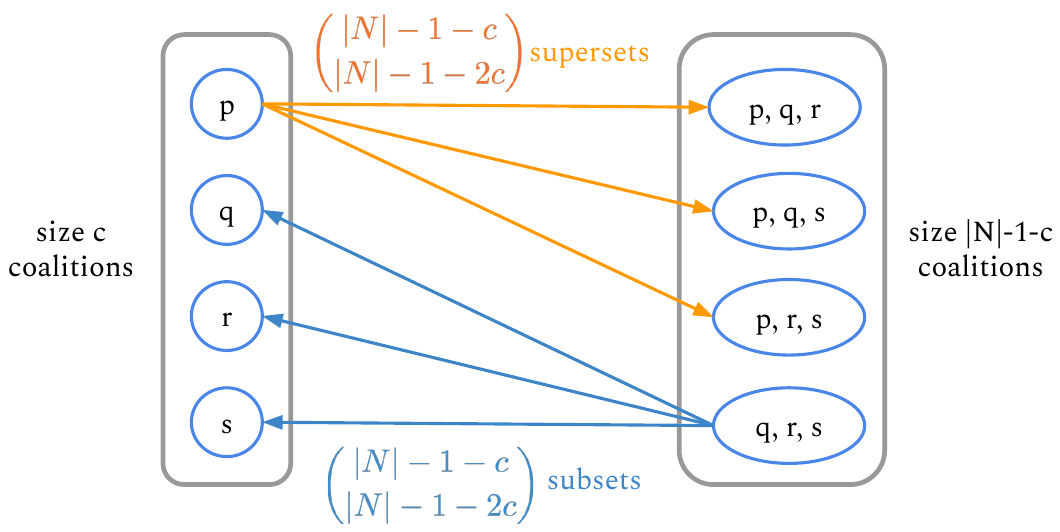}
    \caption{Illustration of the proof for Theorem~\ref{thm:shapleysinglerobustness}. Given an example game with 5 players $N=\{i,p,q,r,s\}$, we match the coalitions (excluding the target player $i$) of size-$c$ and size-$(|N|-1-c)$.  (here $c=1$ and $|N|-1-c=3$). Specifically, each size-$c$ coalition in $L$ (left) has $\binom{|N|-c-1}{|N|-2c-1}$ supersets in $R$ (right),  each size-$|N|-c-1$ coalition in (R) has $\binom{|N|-c-1}{|N|-2c-1}$ subsets of size-$c$. Arrows indicate the set inclusion relations.}
    \label{fig:halls}
\end{figure}
The graph is $k$-regular where every vertex has the same degree $k=\binom{|N|-1-c}{|N|-1-2c}$. To see this, we first show that each coalition $\S_1\in L$ has $\binom{|N|-1-c}{|N|-1-2c}$ supersets in $R$. To find a size $|N|-1-c$ coalition $\S_2\in R$ that is a superset of $\S_1$, we can add $|N|-1-2c$ players by choosing from the remaining $|N|-1-c$ players, i.e., $N\setminus\{i\}\setminus\{\S_1\}$. Therefore, there are $\binom{|N|-1-c}{|N|-1-2c}$ choices and hence the same number of supersets. Similarly, we can show that each coalition $\S_2\in R$ has $\binom{|N|-1-c}{|N|-1-2c}$ subsets in $L$, by removing $|N|-1-2c$ members. Having shown that the $B_c$ is $k$-regular, by Hall’s Marriage Theorem for regular graphs, there exists a perfect matching on $B_c$ and hence a bijective mapping $f$.
Finally, we pair the terms according to $f$: 

Let $\S^c=\{\S\subseteq N\setminus\{i\} \mid |\S|=c\}$ denote all size $c$ coalitions excluding player $i$, \\\\
\resizebox{1.\hsize}{!}{
\begin{minipage}{\linewidth}
\begin{align*}
&\delta\vshapley_i = \sum_{\S\subseteq N \setminus \{i\}} \frac{|\S|!(|N|-|\S| -1)!}{(|N|+1)!} (|N|-2|\S|-1)\mc_i(\S) \\
&= \sum_{0\leq c< \frac{|N|-1}{2}} \frac{c!(|N|-c -1)!(|N|-2c-1)}{(|N|+1)!} \left(\sum_{\S_1\in \S^c}\mc_i(\S_1) - \!\!\!\!\sum_{\S_2\in \S^{|N|-1-c}}\!\!\!\!\mc_i(\S_2) \right)\\
&=\sum_{0\leq c< \frac{|N|-1}{2}} \frac{c!(|N|-c -1)!(|N|-2c-1)}{(|N|+1)!} \sum_{\S_1\in \S^c} \underbrace{\Big(\mc_i(\S_1) - \mc_i(f(\S_1))\Big)}_{\geq 0 \textnormal{ due to submodularity }} \geq 0.
\end{align*}
\end{minipage}}
And this concludes our proof that under the Shapley value, the players can gain a higher total payoff by replication.
\end{proof}

\subsection{Robustness of the Banzhaf Value}
Theorem~\ref{thm:shapleysinglerobustness} shows that the Shapley value is not robust against replication if the player replicates its resource and acts as two players. In what follows we will show that under the same replication manipulation, the Banzhaf value is neutral.

\begin{restatable}{theorem}{banzhafsinglepayoff}\label{thm:banzhafsinglepayoff}
	Let G= (N, \cv) be a submodular game with replication redundant characteristic function $\cv$, a player $i\in N$ replicates and obtains the total payoff as two identities $\S^R = \{i_1, i_2\}$ in the new game $G^R= (N^R, \cv^R)$. Under payoff allocation using the Banzhaf value, the changes in total payoff of player $i$ by replicating is zero, i.e.,
	$\delta \vbanzhaf_i = 0$
\end{restatable}

\begin{proof}
The neutrality of the Banzhaf value under the replication is a natural consequence of the weights defined on the coalitions. It is also closely related to the 2-efficiency axiom, where the Banzhaf value is neutral to the merging or splitting of two players. 
The complete proof is in Appendix~\ref{appendix:banzhafsinglepayoff}.
\end{proof}


In comparison, when the player replicates and acts as two identities, the Shapley value is not replication robust, while the Banzhaf is neutral. This raises a few interesting questions: (1) What governs the robustness of the solution concepts which lead to the different behaviours between the Shapley and Banzhaf values? (2) Can we draw the same conclusion for more than one replications, for example, is the Banzhaf value neutral to an arbitrary number of replications? 
To answer these questions, we next examine the wider class of solution concepts, i.e., semivalues~\cite{dubey1981value}, which include both the Shapley value and Banzhaf value. More importantly, we extend our results to the more general case where the player performs an arbitrary number ($k\geq 1$) of replications.


\section{Replication-robustness of General Semivalues with $k\geq1$ Replications}\label{sec:k>1}

In many real-world applications, the details of replication are only private to the malicious player due to anonymity. Take the online social networks for an example, the digital identities of a player is typically private and accessible to the player itself, and a single player can create multiple false identities. Therefore, it is important to account for the case of an arbitrary number of replications where $k$ is unknown. However, with an arbitrary number of replications, the changes in total payoff no longer exhibit the structured form which allows for coalition pairing. Therefore, to analyse the robustness of the semivalues under $k\geq 1$ replications, we take the following steps (e.g., \emph{\thesection-A} refers to Section~\ref{subsec:zc}): 
\begin{enumerate}[label=\emph{\thesection-\Alph*.}]
    \item represent semivalues as an importance weighted sum of average marginal contributions across coalition sizes,
    \item transform the submodularity into an inequality on the average marginal contributions across coalition sizes,
    \item express the total payoff of the malicious player after replication as (new) importance weighted sum on the (original) average marginal contributions,
    \item we present the conditions on the importance weights which lead to replication robustness,
    \item use the above robustness conditions to evaluate a given semivalue such as the Shapley value. 
\end{enumerate}

\subsection{Semivalues as Weighted Average Marginal Contributions}\label{subsec:zc}
As the first step, we introduce the semivalues~\cite{dubey1981value}, a wide class of Shapley-like solution concepts including both the Shapley and Banzhaf value. The semivalue of a player can be defined as a weighted sum over its marginal contributions towards coalitions of other players. The weights of player $i$'s marginal contribution towards coalition $\S$ is denoted by $\weight{\S}{N}$. In particular, $\weight{\S}{N}$ only depends on the size of the coalition $\S$ but not on the players' identities inside the coalition, i.e.,
\begin{equation}\label{eq:ungrouped_semivalue}
\varphi_i(N, v) = \sum_{\S\subseteq N\setminus\{i\}} \weight{|\S|}{N} \mc_i(\S).
\end{equation}
Therefore, by grouping together equal-sized coalitions, a semivalue assigns to each player $i\in N$ a real-valued payoff, expressed as a weighted sum of player $i$'s \emph{average marginal contributions towards size-$c$ coalitions $z_i(c)$}:
\begin{gather}
\varphi_i(N, v) = \sum_{c=0}^{N-1}\alpha_c z_i(c), \quad \textnormal{where} \label{eq:zc}\\
\begin{aligned}
   z_i(c) &= {\tbinom{|N|-1}{c}}^{-1}\sum_{\S\subseteq{N\setminus\{i\}}, |\S| = c} \mc_i(\S)  \nonumber\\
  \alpha_{c} &= \tbinom{|N|-1}{c}\weight{c}{N} \nonumber \quad \textnormal{(Importance Weights)}
\end{aligned}
\end{gather}

\begin{proof}[Proof Sketch]
The derivation from Equation~\eqref{eq:ungrouped_semivalue} to \eqref{eq:zc} is straightforward and can be obtained by grouping the marginal contributions of player $i$ towards equal-sized coalitions. The normalisation factor $\tbinom{|N|-1}{c}$ is the number of size-$c$ coalitions of players excluding $i$. The proof is in Appendix~\ref{appendix:zc}.
\end{proof}

We will refer to $\alpha_c$ as \emph{importance weights}, as they quantify the importance of a player's marginal contributions towards different coalition sizes. In addition, the importance weights in a semivalue form a probability distribution, that is, $\sum_{c=0}^{|N|-1}\alpha_c = 1$.
The next corollary presents the importance weights of some common semivalues, namely, the Shapley value, Banzhaf value, and Leave-one-out value.

\begin{corollary}[Importance Weights for Common Semivalues]
The Shapley value is defined by the weights $\weight{c}{N} = \frac{c!(|N|-1-c)!}{|N|!} = \frac{1}{|N|}\tbinom{|N|-1}{c}^{-1}$, hence the importance weights are uniform across all coalition sizes, i.e., $\alpha_c^\textnormal{Shapley} = \tbinom{|N|-1}{c}\weight{c}{N} = \frac{1}{|N|}$. 
In contrast, the Banzhaf value is defined by the weights $\weight{c}{|N|} = \frac{1}{2^{|N|-1}}$, hence the importance weights form a bell shape $\alpha_c^\textnormal{Banzhaf}  = \frac{1}{2^{|N|-1}}\tbinom{|N|-1}{c}$ which favours mid-sized coalitions. Finally, for the Leave-one-out value, $\alpha_c^\textnormal{LOO}  = \mathbbm{1}_{c=|N|-1}$.
\end{corollary}

Intuitively, by adjusting the importance weights $\alpha_c$, a semivalue balances a player's \emph{individual value} and \emph{complementary value}. In particular, putting higher importance on smaller coalitions (larger $\alpha_c$ for smaller $c$) favours the individual value and vice-versa. 
So far the representation of semivalues via importance weights has provided some insights for differentiating the common solution concepts. In the following sections, we will show that this representation has significant implications for understanding the difference in robustness of solution concepts against replication in submodular games.


\subsection{Average Marginal Contributions vs. Coalition Sizes}\label{subsec:zc_decrease}
Intuitively, in a submodular game with diminishing returns, a player tends to be less useful in terms of marginal contribution when contributing towards a larger coalition. Can we formally show this intuition? Unfortunately, this does not always hold true for arbitrary pairs of coalitions: given coalitions $\coalition_1$ and $\coalition_2$ where $|\coalition_1| \leq |\coalition_2|$, there is no direct comparison between a player's marginal contributions towards these two coalitions, only except for when $\coalition_1$ is a subset of $\coalition_2$.
Nevertheless, we can formalise this intuition under \emph{average marginal contributions}. 
We now present in the following a useful property of submodular games that the average marginal contributions $z_i(c)$ decrease with coalition size under the submodularity assumption. 
\begin{restatable}{lemma}{zcmonotone}\label{lemma:zc_monotone}
Given a submodular game, the average marginal contribution $z_i(c)$ of a player $i$ monotonic decreases with coalition size $c$, i.e.,
\begin{equation}
\forall {0\leq c < |N|-1},\quad z_i(c) \geq z_i(c+1).
\end{equation}
\end{restatable}

\begin{figure}[t]
    \centering
    \includegraphics[width=0.8\linewidth]{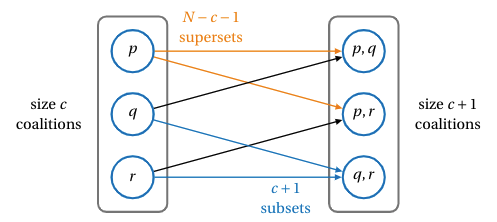}
    \caption{Illustration of the Proof for Lemma~\ref{lemma:zc_monotone}: Given an example game with 4 players $N=\{i,p,q,r\}$, we compare the average marginal contribution of player $i$ towards size-$c$ and size-$(c+1)$ coalitions by matching the coalitions. Specifically, each size-$c$ (here $c=1$) coalition $\S_1$ (Left) has $|N|-c-1$ supersets of size-$(c+1)$. This can be shown by adding any one of the remaining $|N|-c-1$ players ($-c$ refers to the $c$ players already in the coalition and $-1$ refers to the player $i$). Conversely, each size-$(c+1)$ coalition $\S_2$ (Right) has $c+1$ subsets of size-$c$. This can be shown by removing any one of its $c+1$ members. Arrows indicate the "$\subseteq$" relation.}
    \label{fig:lemma_a1}
\end{figure}

\begin{proof} Given player $i\in N$, we show for any coalition size $c$, $z_i(c)\geq z_i(c+1)$, by taking the following steps:


(1) Map the size-$c$ coalitions (excluding $i$) to their size-$(c+1)$ supersets (excluding $i$), and vice-versa: each size-$c$ coalition $\S_1$ can be mapped to $(|N|-1-c)$ number of size-$(c+1)$ supersets $\S_2$ where $\S_1\subseteq{\S_2}\subseteq N\setminus\{i\}$. This can be achieved by adding one of the remaining $(|N|-1-c)$ elements $j\in N\setminus (\{i\}\cup \S_1)$. Conversely, each $\S_2$ can be mapped to $(c+1)$ subsets $\S_1$ of size-$c$. This can be achieved by removing any one of the member elements $j\in \S_2$. An example is shown in Figure~\ref{fig:lemma_a1} for an illustration. 

(2) With the mappings between size $c$ and $c+1$ coalitions, we show that $z_i(c) \geq z_i(c+1)$ by the submodularity property: $\forall \S_1\subseteq \S_2\subseteq N\setminus\{i\} \implies\mc_i(\S_1)\geq \mc_i(\S_2)$. 
The detailed derivations are as follows:
$\forall c\in [0, 1, \ldots,  |N|-2]$, denote $\S^c \coloneqq \{\S\subseteq N\setminus\{i\} \mid |\S| = c\}$ as all possible coalitions of size $c$ excluding player $i$, then\\

\resizebox{1.\linewidth}{!}{
  \begin{minipage}{\linewidth}
\begin{align*}
    &z_i(c+1) - z_i(c) 
        = \sum_{\S_2\in \S^{c+1}}{\tbinom{|N|-1}{c+1}}^{-1}{\mc_i(\S_2)} - \sum_{\S_1\in \S^{c}}{\tbinom{|N|-1}{c}}^{-1}{\mc_i(\S_1)}\\
        &= \sum_{\S_2\in \S^{c+1}}\Big({\tbinom{|N|-1}{c+1}}^{-1}{\mc_i(\S_2)} - \sum_{\S_1\in \S^{c}, \S_1\subseteq{\S_2}}\underbrace{\tfrac{1}{|N|-1-c}}_{(1)}{\tbinom{|N|-1}{c}}^{-1}\underbrace{\mc_i(\S_1)}_{\geq \mc_i(\S_2)}\Big)\\
        &\leq \sum_{\S_2\in \S^{c+1}}\Big({\tbinom{|N|-1}{c+1}}^{-1}{\mc_i(\S_2)} - \sum_{\S_1\in \S^{c}, \S_1\subseteq{\S_2}}\tfrac{1}{|N|-1-c}{\tbinom{|N|-1}{c}}^{-1}{\mc_i(\S_2)} \Big)\\
        &= \sum_{\S_2\in \S^{c+1}}\Big({\tbinom{|N|-1}{c+1}}^{-1}{\mc_i(\S_2)} - \underbrace{\tfrac{c+1}{|N|-1-c}}_{(2)}{\tbinom{|N|-1}{c}}^{-1}{\mc_i(\S_2)}\Big)\\
        &= \sum_{\S_2\in \S^{c+1}}\Big({\tbinom{|N|-1}{c+1}}^{-1}{\mc_i(\S_2)} - {\tbinom{|N|-1}{c+1}}^{-1}{\mc_i(\S_2)} \Big)\\
        &= 0
\end{align*}
\end{minipage}}
$(1)$ $\S_1$ is counted once in each of its $(|N|-1-c)$ supersets $\S_2$ of size-$(c+1)$, and $(2)$ is because each $\S_2$ has $c+1$ subsets $\S_1$ of size-$c$. And this concludes our proof for $z_i(c) \geq z_i(c+1)$.
\end{proof}

We have shown that in a submodular game, a player is more useful \emph{on average} when contributing towards a smaller coalition, i.e., the player's average marginal contribution towards a smaller coalition $z_i(c)$ is no less than its average marginal contribution to a bigger coalition $z_i(c+1)$. With this property, we are ready to extend the replication robustness results to the general class of semivalues and an arbitrary number of replications $k\geq1$.

\subsection{Payoff Changes under Replication with $k\geq 1$}\label{subsec:alpha_ck}
To study the replication robustness of the semivalues, we first derive the total payoff of the replicating player according to the solution concepts after replication. Interestingly, we observe that under the replication redundancy assumption, the replicating player's total payoff can be expressed as a weighted sum of the player's average marginal contributions $z_i(c)$ \emph{from the original game}, as detailed in the following lemma.

\begin{restatable}{lemma}{lemmaack}\label{lemma:ack}
Let $G = (N, v)$ be a submodular game with replication redundant characteristic function $\cv$. By replicating $k$ times and acting as $k+1$ players $\S^R = \{i_0, \ldots, i_k\}$ in the induced game $G^R=(N^R,v^R)$, the malicious player $i$ receives a total payoff of 
\begin{gather}
\varphi_i^\textnormal{tot}(k) = \sum_{c=0}^{|N|-1} \alpha_c^k z_i(c), \label{eq:alpha_ck}
\textnormal{ where }\\
\begin{aligned}
z_i(c) &= {\tbinom{|N|-1}{c}}^{-1}\sum_{\S\subseteq{N\setminus\{i\}}, |\S| = c} \mc_i(\S), \nonumber\\
\alpha_c^k &= (k+1)\tbinom{|N|-1}{c}\weight{c}{N^R}  \quad \textnormal{(new importance weights).}
\nonumber
\end{aligned}
\end{gather}
\end{restatable}
\begin{proof}[Proof Sketch]
By symmetry the replicas yield equal payoff, i.e., $\varphi_i^\textnormal{tot}(k) = (k+1)\varphi_{i_k}(N^R, v^R)$. 
Due to replication redundancy (Assumption~\ref{assumption:replication}), a replica player makes a nonzero marginal contribution only towards coalitions with no other replicas $\S \subseteq N^R\setminus \S^R$, which correspond to the same set of coalitions of the other players in the original game $\S \subseteq N\setminus\{i\}$ because $N^R\setminus \S^R=N\setminus\{i\}$. Following this insight, we can compute the new importance weights $\alpha_c^k$ over the player's original average marginal contributions $z_i(c)$. The complete proof is included in Appendix~\ref{appendix:payoff_changes}.
\end{proof}

Note that Equation~\eqref{eq:alpha_ck} reduces to Equation~\eqref{eq:zc} for no replications, i.e., $\alpha_c^{k} = \alpha_c$ when $k=0$.
Importantly, $z_i(c)$ are the average marginal contributions defined on the original game $G=(N,v)$ as in Equation~\eqref{eq:zc}, instead of on the induced game, thus they \emph{are invariant under replication}. As stated in Equation~\eqref{eq:alpha_ck}, the total payoff of the replicating player is a weighted sum over $z_i(c)$ with the new importance weights $\alpha_c^k$. Since the average marginal contributions $z_i(c)$ in the original game stay invariant after replication, the change in the total payoff of the replicating player $\varphi_i^\textnormal{tot}$ is reflected in the change in $\alpha_c^k$ across different number of replications $k$. This makes $\alpha_c^k$ a key factor for characterising replication robustness. 
The next corollary demonstrates the importance weights after replication for the common semivalues. 

\begin{corollary}[New Importance Weights for Common Semivalues after Replication]
\label{lemma:example_alpha_ck}
After $k$ replications, the new importance weights for the total payoff of the malicious player are: for the Shapley value
$\alpha_c^k = \tfrac{(k+1)\binom{|N|-1}{c}}{(|N|+k)\binom{|N|+k-1}{c}}$, for the Banzhaf value $\alpha_c^k = \frac{(k+1)}{2^{|N|+k-1}}\tbinom{|N|-1}{c}$, and for the Leave-one-out value $\alpha_c^k = \mathbbm{1}_{c = |N| - 1, k = 0}$.
\end{corollary}
\begin{proof}
The new importance weights can be obtained by plugging in the weights of the solution concepts in the induced game to Equation~\eqref{eq:alpha_ck}. 
\end{proof}

\begin{figure}[t]
    \centering
	\includegraphics[width=0.9\linewidth, height=0.4\linewidth]{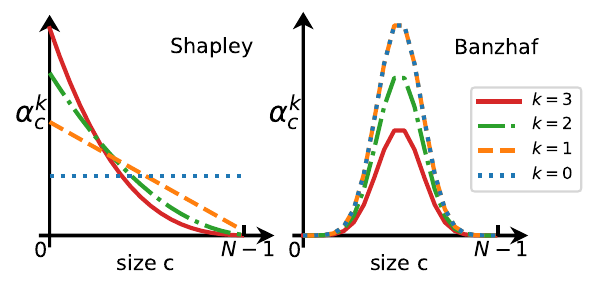}
	\caption{Changes of $\alpha_c^k$ under different number of replications $k$, plot using Equation~\eqref{eq:alpha_ck} with $|N|=20$. The x-axis represents the sizes $c$ of coalitions of the other players $N\setminus\{i\}$, and the y-axis shows the new importance weights $\alpha_c^k$ assigned to each coalition size. Each curve represents a different number of replications $k$.
	(Left) Across the different curves, the importance weights $\alpha_c^k$ of the Shapley value shift towards smaller coalitions as $k$ increases (Lemma~\ref{lemma:sum_alpha_properties}). (Right) In contrast, the Banzhaf importance weights $\alpha_c^k$ are unchanged with the first replication, afterwards, $\alpha_c^k$ decreases across all coalition sizes as $k$ increases.
	Since $z_i(c)$ decreases over coalition size $c$ due to the submodular characteristic function (Lemma~\ref{lemma:zc_monotone}), the weight shift of Shapley value causes $\varphi^\textnormal{tot}_i$ to be increasing, and non-increasing for the Banzhaf value.}
	\label{fig:weights}
\end{figure}

In Example~\ref{example:alpha_ck} and Fig.~\ref{fig:weights}, we compare the Shapley value and the Banzhaf value using Equation~\ref{eq:alpha_ck}, and illustrate the difference between these two solution concepts in terms of their new importance weights after replication.  

\begin{example}[Payoff Changes of the Malicious Player]\label{example:alpha_ck}
Let $G = (N, v)$ be a submodular game with 3 players $N=\{i, p, q\}$. The marginal contributions of player $i$ towards coalitions of other players are $\mc_i(\emptyset) = 3$, $\mc_i(\{p\})= \mc_i(\{q\}) = 2$, $\mc_i(\{p,q\}) = 1$. Player $i$ replicates once and acts under two identities $\S^R = \{i_1, i_2\}$. The induced game is then $G^R = (N^R, v^R)$ where $N^R = \{i_1, i_2, p, q\}$.
To see the changes in $i$'s total payoff, we first compute the average marginal contributions of $i$ in the original game:
 $z_i(0) = \mc_i(\emptyset) = 3;\quad z_i(1) = \frac{1}{2} (\mc_i(p)+ \mc_i(q))= 2;\quad  z_i(2) =\mc_i(p,q)= 1.$
Then we compute the total payoffs $\varphi_i^\textnormal{tot}(k)$ of player $i$ according to the Shapley and Banzhaf value using Equation~\eqref{eq:alpha_ck}, where $k=0$ refers to no replication, and $k>0$ represents replicating $k$ times:

(Shapley)~ $\varphi_i^\textnormal{tot}(0) = \sum_{c=0}^{2}\alpha_c^0 z_i(c) = 3\cdot\tfrac{1}{3}+2\cdot\tfrac{1}{3}+1\cdot\tfrac{1}{3} = 2;$ 

\phantom{(Shapley)} $\varphi_i^\textnormal{tot}(1) = \sum_{c=0}^{2}\alpha_c^1 z_i(c) = 3\cdot\tfrac{1}{2}+2\cdot\tfrac{1}{3}+1\cdot\tfrac{1}{6} = \tfrac{7}{3}$.

(Banzhaf)  $\varphi_i^\textnormal{tot}(0) = \sum_{c=0}^{2}\alpha_c^0 z_i(c) = 3\cdot\tfrac{1}{4}+2\cdot\tfrac{1}{2}+1\cdot\tfrac{1}{4} = 2;$ 

\phantom{(Banzhaf)}$\varphi_i^\textnormal{tot}(1) = \sum_{c=0}^{2}\alpha_c^1 z_i(c) = 3\cdot\tfrac{1}{4}+2\cdot\tfrac{1}{2}+1\cdot\tfrac{1}{4} = 2$.
\end{example}

The above example demonstrates our observations from Fig.~\ref{fig:weights}: the importance weights $\alpha_c^k$ of Shapley value shifts from uniform in the coalition sizes $(\frac{1}{3}, \frac{1}{3}, \frac{1}{3})$, to having larger weights towards smaller coalition sizes $(\frac{1}{2}, \frac{1}{3}, \frac{1}{6})$ after replication. Due to submodularity, the average marginal contributions are larger for smaller coalition sizes, hence the total payoff increases as a result of replication. Whereas for the Banzhaf value, the importance weights $\alpha_c^k = (\frac{1}{4}, \frac{1}{2}, \frac{1}{4})$ are invariant under the first replication, hence the total payoff is unchanged. In the following section we provide a formal characterisation of these observations.

\subsection{Replication Robustness Condition}\label{subsec:robustness_conditions}
In this section, we will present the condition which characterises the replication robustness for semivalues. Specifically, this condition provides a sufficient and necessary condition on the importance weights $\alpha_c^ k$ for guaranteeing replication robustness against any arbitrary number of replications $k$. By using a replication robust solution concept, a player should have no incentive to perform any number of replications in order to increase its payoff.


\begin{theorem}[Replication Robustness Condition]
\label{thm:robustness_condition} 
Given a submodular game with replication redundant characteristic function, a solution concept of the form $\varphi_i = \sum_{c=0}^{|N|-1}\alpha_c z_i(c)$ 
is replication robust \emph{if and only if} for any number of replications $k$,
\begin{equation}\label{eq:robustness_condition}
\forall 0\leq p\leq |N|-1, \sum_{c=0}^p\alpha_c^0 \geq \sum_{c=0}^p\alpha_c^{k},
\end{equation}
where $\alpha_c^k$ are the importance weights as defined in Equation~\eqref{eq:alpha_ck}.
\end{theorem}

\begin{proof}
{\bfseries Sufficiency.}
We will first show the sufficient condition, that is, Equation~\eqref{eq:robustness_condition} implies replication robustness, i.e., $\varphi_i^\textnormal{tot}(0) - \varphi_i^\textnormal{tot}(k)\geq 0$.
Due to submodularity, the average marginal contributions of a player decrease as growing coalition sizes, according to Lemma~\ref{lemma:zc_monotone}. Together with replication redundancy, we have the average marginal contributions satisfy the following condition $z_i(0)\geq\ldots\geq z_i(|N|-1) \geq 0$. 

Let $\delta^k_c = \alpha_c^0  -  \alpha_c^k$ denote the difference in importance weight over coalition size $c$ before and after replication, then by Equation~\eqref{eq:robustness_condition}, $\forall p\in \{0, 1, \ldots , |N|-1\}, \sum_{c=0}^p \delta^k_c \geq 0$. Therefore, we proceed to show the inequality recursively:\\
    $\varphi_i^\textnormal{tot}(0) - \varphi_i^\textnormal{tot}(k)
        = \sum_{c=0}^{|N|-1} \delta^k_c z_i(c)\\
        \phantom{===}= z_i(0)\sum_{c=0}^{0} \delta^k_c + \sum_{c=1}^{|N|-1} \delta^k_c z_i(c) \\
        \phantom{===}\stackrel{(1)}{\geq} z_i(1)\sum_{c=0}^{0} \delta^k_c + \sum_{c=1}^{|N|-1} \delta^k_c z_i(c)  \\
        \phantom{===}=  z_i(1) \sum_{c=0}^{1} \delta^k_c + \sum_{c=2}^{|N|-1} \delta^k_c z_i(c) \\
        \phantom{===}\stackrel{(2)}{\geq} z_i(2)\sum_{c=0}^{1} \delta^k_c + \sum_{c=2}^{|N|-1} \delta^k_c z_i(c) \geq \ldots \\
        \phantom{===}= z_i(|N|-2)\sum_{c=0}^{|N|-2} \delta^k_c + \sum_{c=|N|-1}^{|N|-1} \delta^k_c z_i(|N|-1)\\
        \phantom{===}\geq z_i(|N|-1) \sum_{c=0}^{|N|-2} \delta^k_c + \sum_{c=|N|-1}^{|N|-1} \delta^k_c z_i(|N|-1) \\
        \phantom{===}=  z_i(|N|-1) \sum_{c=0}^{|N|-1} \delta^k_c 
        \geq 0.$

where $(1)$ is because $z_i(0) \geq z_i(1)$ and $\sum_{c=0}^0 \delta_c \geq 0$, and $(2)$ is because $z_i(1) \geq z_i(2)$ and $ \sum_{c=0}^1 \delta_c \geq 0$.
With this, we have shown the sufficient condition, and we will next show the necessary condition.

\noindent
{\bfseries Necessity.}
We now show that Equation~\eqref{eq:robustness_condition} is also a necessary condition. 
To do this, we will prove by contradiction:

Let $\Tilde{\delta}_c^k \coloneqq \Tilde{\alpha_c^{0}} -\Tilde{\alpha_c^{k}}$.
Recall in Equation~\eqref{eq:robustness_condition} for any coalition size $0\leq p\leq |N|-1, \sum_{c=0}^{q} \Tilde{\delta}_c^k \geq 0$. Assume the contrary that there exists a set of coalition sizes (index) $q_m$ where the condition does not hold: $$\exists Q_m=\{q_0, q_1,\ldots, q_m\}, \textnormal{ such that } \forall q\in Q_m, \sum_{c=0}^{q} \Tilde{\delta}_c^k < 0, $$ 
Without loss of generality, we assume the coalition sizes are ordered and that  $q_0<q_1<\ldots<q_m\leq |N|-1$.
We now show that there exist average marginal contributions $z_i(c)$'s which violates replication robustness, and we construct them as follows: Looking at the smallest index that causes the contrary assumption $q_0 = \min Q_m$, all indices below $q_0$ satisfy the original condition, i.e.,
\begin{align*}
    &\sum_{c=0}^p \Tilde{\delta}_c^k 
        \begin{cases}   < 0 & \textit{ if } p= q_0\\
                        \geq 0 & \textit{ if } p< q_0
        \end{cases}
\end{align*}

Therefore, the sum of $\Tilde{\delta_c^k}$ over all indices under $q_0$ is less than the absolute value of that at $q_0$, i.e., $0 \leq \sum_{c=0}^{q_0-1}\Tilde{\delta}^k_c < - \Tilde{\delta}^k_{q_0} = |\Tilde{\delta}^k_{q_0}|$.
Therefore, we denote $\gamma <1$ as the ratio, such that
$\sum_{c=0}^{q_0-1} \Tilde{\delta}_c^k = \gamma |\Tilde{\delta}_{q_0}^k|$.
To construct $z_i(c)$, we let $\forall c > q_0, z_i(c) = 0$ for all indices above $q_0$,  and let  $z_i(q_0) = \gamma z_i(0) + \epsilon$ where $0<\epsilon \leq (1-\gamma)z_i(0)$. Note that the $\epsilon>0$ is for $z_i(q_0)$ to be strictly greater than $\gamma z_i(0)$, and $\epsilon \leq (1-\gamma)z_i(0)$ guarantees submodularity where $z_i(q_0) \leq z_i(0)$. In fact, a trivial choice would be a constant function for all indices no greater than $q_0$, i.e., $\forall q\in\{0, \ldots , q_0\}, z_i(q) = Const.$, but we will adopt the former option which also accounts for strictly submodular cases. Then, we have\\
$ \Tilde{\varphi}_i^\textnormal{tot}(0) - \Tilde{\varphi}_i^\textnormal{tot}(k) 
    = \sum_{c=0}^{|N|-1}\Tilde{\delta}_c^k z_i(c)\\
    \phantom{===}\stackrel{(1)}{=} \sum_{c=0}^{q_0}\Tilde{\delta}_c^k z_i(c)
    = (\sum_{c=0}^{q_0-1}\Tilde{\delta}_c^k z_i(c)) + \Tilde{\delta}_{q_0}^k z_i(q_0)\\
    \phantom{===}\stackrel{(2)}{\leq} (\sum_{c=0}^{q_0-1}\Tilde{\delta}_c^k) z_i(0) + \Tilde{\delta}_{q_0}^k z_i(q_0)
    = |\Tilde{\delta}_{q_0}^k| (\gamma z_i(0) - z_i(q_0))\\
    \phantom{===}= |\Tilde{\delta}_{q_0}^k| (\gamma z_i(0) - \gamma z_i(0) - \epsilon)
    = -\epsilon|\Tilde{\delta}_{q_0}^k|
 < 0
$,

\noindent
where $(1)$ is due to $\forall q_0< q \leq |N|-1,  z_i(q)=0$, and $(2)$ is due to submodularity.

This forms a contradiction. Thus we have shown that Theorem~\ref{thm:robustness_condition} is both a necessary and sufficient condition for replication robustness, and this concludes our proof.
\end{proof}

Intuitively, the condition ensures that after replication, there is not significant increase in importance weight on the small coalition sizes, towards which a player has larger average marginal contributions $z_i(c)$ due to submodularity. This effect was illustrated in Fig.~\ref{fig:weights} and the theorem is a formal characterisation.
The significance of this theorem is that it provides a necessary and sufficient condition for guaranteeing replication robustness for all semivalues and for any number of replications. Therefore, a solution concept that satisfies the condition is replication robust without the need of knowing the number of replications $k$ or the replicated false identities.
Extending from the necessary and sufficient condition in Theorem~\ref{thm:robustness_condition}, the following two corollaries provide sufficient conditions for monotonic decreasing (Corollary~\ref{thm:monotonicity_condition}) and monotonic increasing (Corollary~\ref{thm:monotonicity_condition_shapley}) total payoff of the replicating player with respect to the number of replications. These conditions will help us characterise the robustness of the common semivalues.

\begin{corollary}[Monotonic Decreasing Total Payoff]
\label{thm:monotonicity_condition}
Given a submodular game with a replication redundant characteristic function, a semivalue is replication robust and the total payoff of the malicious player decreases monotonically i.e.,$\varphi_i^\textnormal{tot}(k) \! \geq \! \varphi_i^\textnormal{tot}(k+1)$, \emph{if} for any number of replications $k$,
\begin{equation}\label{eq:robustness_sufficient_condition}
\forall 0\leq p\leq N-1, \sum_{c=0}^p \alpha_c^{k}  \geq  \sum_{c=0}^p  \alpha_c^{k+1} 
\end{equation}
\end{corollary}
Note that the condition stated in Corollary~\ref{thm:monotonicity_condition} is stricter than that in Theorem~\ref{thm:robustness_condition} which implied $\varphi_i^\textnormal{tot}(0)\geq \varphi_i^\textnormal{tot}(k)$, but additionally ensures that the total payoff of a replicating player monotonic decreases with the number of replications, i.e., $\forall k\geq 0, \varphi_i^\textnormal{tot}(k) \! \geq \! \varphi_i^\textnormal{tot}(k+1)$.
\begin{proof}[Proof Sketch]
We need to show that Equation~\eqref{eq:robustness_sufficient_condition} implies the total payoff decreases with $k$:
$$
\forall k, \varphi_i^\textnormal{tot}(k) - \varphi_i^\textnormal{tot}(k+1)
= \sum_{c=0}^{|N|-1} (\alpha_c^{k}- \alpha_c^{k+1})z_i(c) \geq 0,
$$
\noindent
Denote $\delta_c^k \coloneqq \alpha_c^{k} - \alpha_c^{k+1}$, then we can substitute $\delta_c^k$ in the recursive proof for the sufficient condition of Theorem~\ref{thm:robustness_condition}, by doing so we will reach the above conclusion.
\end{proof}

\begin{corollary}[Monotonic Increasing Total Payoff]
\label{thm:monotonicity_condition_shapley}
Given a submodular game with a replication redundant characteristic function, a semivalue is not replication robust, and the total payoff of the malicious player increases monotonically i.e.,$\varphi_i^\textnormal{tot}(k)  \leq \varphi_i^\textnormal{tot}(k+1)$, \emph{if} for any number of replications $k$,
\begin{equation} \forall 0 \leq p \leq N-1, \sum_{c=0}^{p} \alpha_c^{k+1} \geq \sum_{c=0}^{p} \alpha_c^{k}\label{eq:shapley_alpha_ck}
\end{equation}
\end{corollary}
\begin{proof}[Proof Sketch]
We need to show that Equation~\eqref{eq:shapley_alpha_ck} implies the total payoff increases with $k$:
$$
\forall k, \varphi_i^\textnormal{tot}(k+1) - \varphi_i^\textnormal{tot}(k) 
= \sum_{c=0}^{|N|-1} (\alpha_c^{k+1} - \alpha_c^{k})z_i(c) \geq 0,
$$
The proof is similar to the monotonic increasing case above. Denote $\delta_c^k \coloneqq \alpha_c^{k+1} - \alpha_c^k$, and we can reuse the proof for Theorem~\ref{thm:robustness_condition} to reach the above conclusion.\end{proof}


\subsection{Robustness of Common Solution Concepts}\label{subsec:robustness_common}
Now using the robustness conditions presented in the previous section, we can revisit the Shapley value and the Banzhaf value, as well as Leave-one-out with $k\geq 1$ replications. The following theorem is a consequence of our robustness condition for the three common semivalues.

\begin{restatable}[]{theorem}{shapleysubmodular}
\label{thm:gain_increase}
Let $G=(N,v)$ be a submodular game where $v$ is replication redundant, the Shapley value is not replication robust, whereas the Banzhaf value and Leave-one-out are replication robust.
For the Shapley value, the total payoff of the replicating player $i$ monotonic increases over the number of replicas, and converges to $i$'s characteristic value, i.e.,
    $\lim_{k\rightarrow{\infty}}\varphi^\textnormal{tot}_{i}(k) = v(\{i\}).$ For the Banzhaf and Leave-one-out values, $\lim_{k\rightarrow{\infty}}\varphi^\textnormal{tot}_{i}(k) = 0.$
\end{restatable}

\begin{proof}[Proof] \textbf{For the Shapley Value.} We prove that the Shapley value is not replication robust, and the total payoff of the replicating player monotonic increases with growing $k$ in the following three steps: 

1. Express the Shapley value after replication according to Lemma~\ref{lemma:ack} in terms of average marginal contributions and importance weights, i.e.,
 $\alpha_c^k = \frac{(k+1)}{(|N|+k)}\binom{|N|-1}{c}\binom{|N|+k-1}{c}^{-1}.$

2. In Lemma~\ref{lemma:sum_alpha_properties} (presented following this theorem), we show that the Shapley value satisfies Equation~\eqref{subeq:monotonicity}:
$$\forall 0\leq p \leq N-1, \quad \sum_{c=0}^p\alpha_c^k \leq \sum_{c=0}^p\alpha_c^{k+1}.$$

3. By Theorem~\ref{thm:robustness_condition}, the Shapley value is not replication robust. In addition, Corollary~\ref{thm:monotonicity_condition} shows that for the Shapley value, the total payoff of the replicating player monotonic increases with respect to increasing number of replications $k$.

Finally, the limit is computed as follows:

\begin{align*}
\lim_{k\rightarrow{\infty}}\varphi_i^\textnormal{tot}(k) 
    &= \lim_{k\rightarrow{\infty}}\sum_{c=0}^{|N|-1}\tfrac{k+1}{|N|+k}\tbinom{|N|-1}{c}{\tbinom{|N|+k-1}{c}}^{-1} z_i(c)\\
    &= \sum_{c=0}^{|N|-1}\tbinom{|N|-1}{c}z_i(c)\underbrace{\lim_{k\rightarrow{\infty}}\!\tfrac{k+1}{|N|+k}}_{=1} \underbrace{\lim_{k\rightarrow{\infty}}\!{\tbinom{|N|+k-1}{c}}^{-1}}_{=\mathbbm{1}_{c=0}}\\
    &= z_i(0) = \mc_i(\emptyset) = v(\{i\})
\end{align*}
Proofs of the Banzhaf value and Leave-one-out value are in Appendix~\ref{appendix:gain_increase}.
\end{proof}

We observe that due to replication redundancy and submodularity, the solution concepts which emphasize the complementary value tend to be more replication robust. 

The robustness property of the Shapley value generalises our findings in Section~\ref{subsec:k=1_shapley} for the $k=1$ case. With an increasing number of replications, the player's total payoff monotonic increases and converges to its own characteristic value. This is due to the following properties shown in the next lemma.

\begin{restatable}{lemma}{shapleyproperties}
\label{lemma:sum_alpha_properties} 
For the Shapley value, the importance weights $\alpha_c^k$ of the total payoff of a replicating player satisfy the following properties: $ \forall k\geq 0, \forall 0\leq p\leq |N|-1, $
\begin{subequations}
\begin{align}
 &\sum_{c=0}^{|N|-1}\alpha_c^k = 1\label{subeq:efficiency}\\
 &\sum_{c=0}^p\alpha_c^k \leq \sum_{c=0}^p\alpha_c^{k+1}\label{subeq:monotonicity}\\
&\sum_{c=0}^p\alpha_c^{k+1} - \alpha_c^k \geq \sum_{c=0}^p\alpha_c^{k+2}- \alpha_c^{k+1}\label{subeq:submodularity}
\end{align}
\end{subequations}
\end{restatable}

\begin{proof}
The complete proof is included in Appendix~\ref{appendix:shapley_replication}.
\end{proof}

Equation~\eqref{subeq:efficiency} describes an interesting phenomenon that the new importance weights of the Shapley value after replication (i.e., $\alpha_c^k$) always sum to 1. This is a special property of the Shapley value which is not shared by all semivalues. In contrast to the property of semivalues, i.e., $\sum_{c=0}^{|N|-1}\alpha_c = 1$, Equation~\eqref{subeq:efficiency} states that the sum of the new importance weights after replication $\alpha_c^k$ over the coalitions of honest players in the original game always sum to $1$. Moreover,
Equation~\eqref{subeq:monotonicity} shows that these importance weights gradually \emph{shifts} towards the smaller coalitions with each added replication, which results in the monotonic increasing total payoff. Collectively, Equation~\eqref{subeq:efficiency} and Equation~\eqref{subeq:monotonicity} results in the convergence of the malicious player's total payoff to its characteristic value.
Additionally, Equation~\eqref{subeq:submodularity} implies that the gain of adding one replica decreases with replication, hence the first replication yields the highest unit gain.

To summarise, we have analysed and compared the replication robustness of the common semivalues, namely, the Shapley value, the Banzhaf value and the Leave-one-out value. In the following section, we discuss the design of other replication robust solution concepts.

\subsection{Other Replication Robust Payoff Allocations}\label{sec:robustness_solutions}
In this section, we describe how to apply the replication robustness condition to find other robust solution concepts and illustrate this with an example robust solution concept derived from the Shapley value.

\emph{Observation 1:} To satisfy the robustness conditions in Theorem~\ref{thm:robustness_condition}, it suffices to satisfy one of the following conditions for each summand of coalition size $c$:
\begin{align}\label{eq:robustness_sufficient_condition_summand}
&\alpha_c^{0} \geq \alpha_c^{k},
\textit{ or monotonicity: }\alpha_c^{k} \geq \alpha_c^{k+1}
\end{align}

\emph{Observation 2:} The identity of the replicating player and the number of replicas $k$ are private information that is often not accessible. Therefore, we should make sure that $k$ does not appear in the solution concept.

We now derive a robust solution by down-weighing the Shapley value using these two observations. Our solution will take the following form, where the factor $\gamma_{|N|}^{|\S|}$ is a function of the total number of players $|N|$ and coalition size $|\S|$:
\begin{align}
  &\Tilde{\varphi}_i(N, v) \coloneqq \sum_{\S\subseteq N \setminus \{i\}} \gamma_{|N|}^{|\S|} \weight{|\S|}{N} \mc_i(\S) \label{eq:downweighed-shapley}
\end{align}
where $\weight{|\S|}{N} = \frac{|\S|!(|N|-|\S| -1)!}{|N|!}$ are the Shapley coefficients.

\begin{restatable}{definition}{thmGamma}{(Robust Shapley value)} Equation~(\ref{eq:downweighed-shapley}) with 
\label{thm:gamma}
$$
    \gamma_{|N|}^{|\S|} = \begin{cases}
            \frac{\lceil \tfrac{|N|-1}{2}\rceil!\lfloor\tfrac{|N|-1}{2}\rfloor!}{|\S|!(|N|-|\S|-1)!} &\textit{if } |\S| < \lfloor{\tfrac{|N|-1}{2}}\rfloor{},\\
             1 &\textit{otherwise}.
    \end{cases}
$$
defines the Robust Shapley value.
\end{restatable}
\begin{restatable}{corollary}{corollaryloss}{}\label{corollary:robust_shapley_loss}
    The Robust Shapley value is replication robust. Moreover, in a submodular game $G=(N, v)$, the loss for a replicating player $i$ by replicating $k$ times 
		$
		\varphi^\textnormal{tot}_i(0) - \varphi^\textnormal{tot}_i(k)\geq \tfrac{1}{|N|}\sum_{c=0}^{|N|-1}(1-\tfrac{k+1}{2^k})\gamma_{|N|}^c z_i(c)
		$.
\end{restatable}

\begin{proof}
The complete proof is included in Appendix~\ref{appendix:robust_shapley}
\end{proof}

The Robust Shapley value satisfies axioms symmetry~\ref{ax:symmetry}, null-player~\ref{ax:nullplayer}, linearity~\ref{ax:linearity}. Additionally, the total allocated payoff does not exceed the value of the grand coalition.

Like the Banzhaf value and Robust Shapley value, there are many other possible solution concepts which are replication robust. These solution concepts can be crafted by designing the importance weights. 
Recall that semivalues balance between a player's individual value and complementary value through the importance weights.
As a rule of thumb, the solution concepts which emphasize the complementary value and put larger importance weights on the mid-sized and larger coalitions tend to be more replication robust.

\input{sections/related_attacks}

\input{sections/facility_location}

%% file: sections/related_attacks.tex
\subsection{Perturbed Replication}\label{app_sub:attack}
Sometimes, the manipulations may not be an  exact replication. We now consider a related scenario where the malicious player replicates its resources and splits into multiple identities, then perform a small perturbation on its replicated resources to avoid replica detection, such as adding noise. 
A perturbed replication can be formulated as follows: In the submodular game $G=(N, v)$, a malicious player $i$ replicates its resources $D_i$ $k$ times and acts as $k+1$ players $\S^R = \{i_0, \ldots, i_k\}$ where $D_{i_k} = D_i$. The player further perturb its replicas as $\S^P = \{p_0, \ldots, p_k\}$, where $D_{p_k} = f_k(D_i)$ for some perturbation function $f_k$. The malicious player receives a total payoff as a sum of all its perturbed replicas, i.e., $\varphi_i^\textnormal{replicate} = \sum_{i_k\in \S^R}\varphi_{i_k}$ and  $\varphi_i^\textnormal{perturb} = \sum_{p_k\in \S^P}\varphi_{p_k}$.
Assume the effect of perturbations are small such that (1) the marginal contribution of the perturbed replicas towards the other players remain unchanged, that is: \begin{equation*}
\forall{\S\subseteq N\setminus\{i\}}, \mc_{p_k}(\S) = \mc_{i_k}(\S),
\end{equation*}
and (2) the marginal contributions of each perturbed replica towards coalitions containing other perturbed replicas are small, that is, there exists a small quantity $\exists \epsilon >  0$ s.t., 
$$\forall{p_k \in \S^P,\emptyset \neq \S^p\subseteq \S^P\setminus\{p_k\}}, \S\subseteq N\setminus\{i\}, \mc_{p_k}(\S^p\cup\S)\leq \epsilon $$

\begin{restatable}{lemma}{perturbation}\label{lemma:perturbation}
 Compared with replication, the additional gain in total payoff of the malicious player due to the perturbation when replicating $k$ times is given by: \begin{equation*}\varphi_i^\textnormal{perturb}- \varphi_i^\textnormal{replicate} \leq (k+1)\epsilon.
 \end{equation*}
\end{restatable}
\begin{proof}
The proof is included in Appendix~\ref{appendix:attack}.
\end{proof}

In this way, perturbations which yield negligible marginal values towards other players and the other perturbed replicas will yield negligible benefit compared with the non-perturbed replicas. Therefore, the replication robust solution concepts are also $\epsilon-$robust against the perturbed replication manipulation. 

%% file: sections/facility_location.tex
\subsection{Robustness Results on the Facility Location Game}\label{subsec:facility_location_robustness}
\input{sections/facility_table}
Having presented the theoretical results on replication robustness, we now demonstrate these findings on the facility location game as defined in Example~\ref{example:facility_location_game}. Before that, we present the following theorem, which efficiently computes the Shapley and Banzhaf value in the facility location game, allowing us efficiently visualise their convergence properties.
\begin{restatable}{theorem}{facilityshapley}\label{thm:facility_location_shapley}
    The Shapley and Banzhaf value of a facility location $i$ in a facility location game can be computed as
    \begin{align*}
      \vshapley_i &= \sum_{d\in D} \big[ u_{id} \frac{1}{|\loc|-|\loc_{id}|} -\sum_{t=1}^{|\loc_{id}|} \frac{1}{\lambda(t)+{\lambda(t)}^2} u_{e_{it}^d} \big],\\
      \vbanzhaf_i &= \frac{1}{2^{|\loc|-1}}\sum_{d\in D} \big[ 2^{|\loc_{id}|} 
      u_{id} - \sum_{t=1}^{|\loc_{id}|} 2^{|\loc_{id}| - t} u_{e_{it}^d} \big],
  \end{align*}
    where $\lambda(t)\coloneqq (|\loc|-|\loc_{id}|+t-1)$, $\loc_{id}\coloneqq\{j \in \loc \mid u_{jd} \leq u_{id}\}$ and $u_{e_{it}^d}$ is the utility value of the $t$-th largest element after $i$ along the dimension (customer) $d$, $D$ is the set of customers, $u_{id}$ is the utility of a customer $d$ from facility location $i$.
\end{restatable}

\begin{proof}[Proof Sketch]
    Denote $w_{\S}$ as the weight assigned by the Shapley (Banzhaf) value to coalition $\S$. We observe that
    \resizebox{1.\linewidth}{!}{$
      \varphi_i = \sum_{\S \subseteq \loc \setminus\{i\}} w_{\S} \mc_i(\S)
         \stackrel{(*)}{=} \sum_{d\in D} \big[ \underbrace{\sum_{\S \subseteq \loc_{id}} w_{\S} u_{id}}_{(\#1)} - \underbrace{\sum_{\S \subseteq \loc_{id}} w_{\S} \max_{j \in \S} u_{jd}}_{(\#2)} \big],\nonumber$
    }
    where $(*)$ is because the marginal contribution of $i$ to coalition $\coalition$ in dimension (customer) $d$ is zero \emph{unless} $i$ is the largest element in $\coalition$ in the $d$-th dimension, i.e., subsets of $\loc_{id}=\{j \in \loc \mid u_{jd} \leq u_{id}\}$. Along each dimension $d$, $(\#1)$ is a weighted sum of $i$'s marginal contributions towards coalitions $\S$ where $i$ is the largest element; $(\#2)$ sums up for each $j\in \loc_{id}$ over 
    coalitions $\S\subseteq \loc_{id}$ where $j$ is the largest element. The proof is included in Appendix~\ref{appendix:facility_location}.
\end{proof}


The Shapley and Banzhaf value of the facility location game can be computed using Lemma~\ref{thm:facility_location_shapley}, which can be implemented efficiently by sorting the facility location utility matrix along each dimension (customer), as summarised in Algorithm~\ref{algo:fast_shapley} in Appendix~\ref{appendix:fast_shapley}. 
  Table~\ref{table:facility_location} demonstrates that our algorithm (ours) significantly improves the computation efficiency compared with the naive algorithm (naive) which enumerates all possible coalitions. The output values computed by both algorithms are verified to be equal. We observed that the naive algorithm struggle in games with large number of players (e.g., $n\geq50$) while ours scales up easily. With the help of Algorithm~\ref{algo:fast_shapley}, we can efficiently visualise the Shapley and Banzhaf value in Fig.~\ref{fig:facility_location_results} and validate their robustness and convergence properties under replication in Fig.~\ref{fig:replicate_facility}.

\begin{figure}[t]
    \centering
    \includegraphics[width=1.0\linewidth]{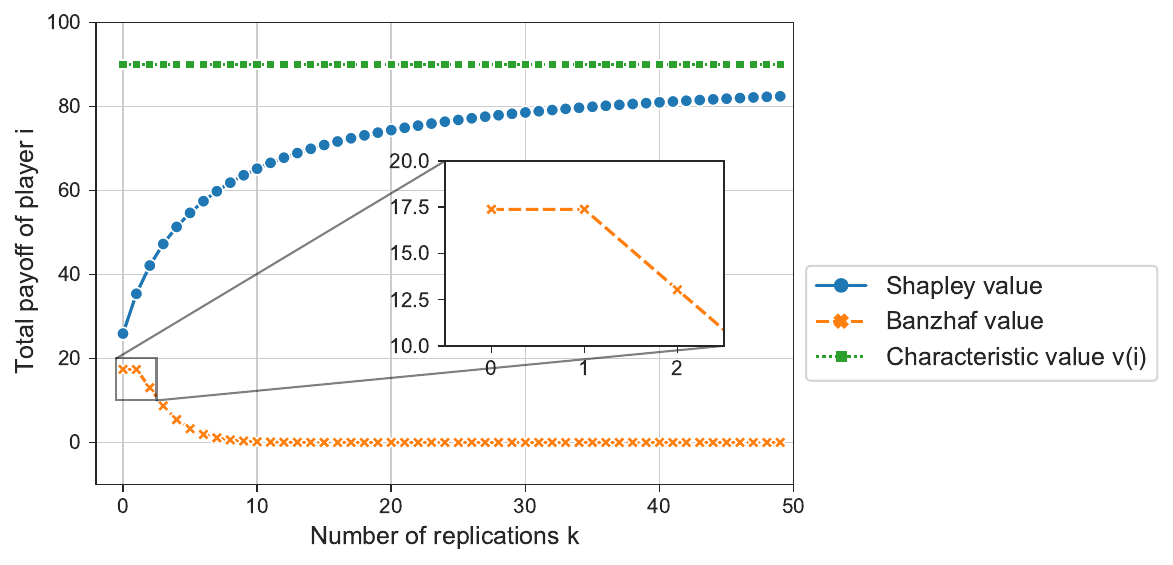}
    \caption{Replication robustness of the Shapley and Banzhaf value for the Facility Location Game}
    \label{fig:replicate_facility}
\end{figure}

Fig.~\ref{fig:replicate_facility} shows the replication robustness of the Shapley value and the Banzhaf value in the facility location game. Specifically, the original facility location game includes $|\loc|=10$ players (facility locations) and $|D|=10$ customers, each utility value $u_{ij}$ is an integer uniformly sampled from $[0, 20]$. Among the players, a malicious player $i$ replicates itself $k$ times and acts under $k+1$ identities. The curves show the total payoff of the player $i$ with respect to growing number of replications $k$ from $0$ to $50$. 
The graph validates some of our findings. Specifically, the Shapley value is not replication robust: the total payoff of the player monotonically increases, and converges to the characteristic value of the player, i.e., $\cv(\{i\})$. Moreover, the unit gain of the player for adding each player monotonically decreases. This can be seen from the decreasing height between pairs of adjacent points. In contrast, the Banzhaf value is replication robust: the total payoff of the replicating player monotonic decreases, and converges to $0$. Moreover, by comparing $k=0$ and $k=1$ (zoomed), we can see that the Banzhaf value is neutral to the first replication.

%% file: sections/facility_table.tex

\begin{table}[t]
\caption{Computation Time of the Shapley and Banzhaf value in the Facility Location Game (in seconds) for $n$ players using the naive approach and our algorithm (Algorithm~\ref{algo:fast_shapley}).}
\centering
\begin{tabular}{@{}ccccccc@{}}
\toprule
 &  & \textbf{n=10} & \textbf{n=15} & \textbf{n=20} & \textbf{n=50} & \textbf{n=100} \\ \midrule
\multirow{2}{*}{\textbf{Shapley value}} & \textit{\textit{naive}} & 0.283 & 14.182 & 558.525 & - & - \\
 & \textit{\textit{ours}} & 0.004 & 0.007 & 0.011 & 0.099 & 0.225 \\ 
\multirow{2}{*}{\textbf{Banzhaf value}} & \textit{\textit{naive}} & 0.245 & 12.623 & 482.803 & - & - \\
 & \textit{\textit{ours}} & 0.002 & 0.003 & 0.006 & 0.056 & 0.170 \\ \bottomrule
\end{tabular}
\label{table:facility_location}
\end{table}

%% file: sections/market.tex
Our theoretical results can be applied to study redundancy and payoff allocation in many submodular real-world ML applications such as multiagent sensing, feature importance evaluation and multi-party ML. In this section, we investigate payoff allocation for ML data markets~\cite{agarwal2019marketplace, ohrimenko2019collaborative} -- an emerging application which readily connects data buyers (i.e., ML practitioners) with data sellers, providing a nice alternative to addressing the challenge of data acquisition in real-world ML applications. 
A naive implementation of such a market in the form of a direct data exchange is likely to fail in practice as data can be freely replicated, and hence may be easily be resold by a buyer. Moreover, acquiring ownership of a large dataset may exceed the budget of the buyer. 
These issues can be alleviated by modelling the market as an integral part of a cloud ML platform: At each round of interaction, the buyer provides a classification task, specified by a validation dataset $D_\textnormal{val}$. The data from multiple data sellers will be pooled securely to jointly train a model $\model(\cup_{i\in N} D_i)$ towards the classification task. The buyer will then pay a fee according to the performance of the model, and the sellers are allocated a payoff according to their data's contributions.

In the following, we model the market as a submodular game, and apply our theoretical insights to study robust payoff allocation against replication, i.e., data replication attacks.

\subsection{Data Market as a Submodular Game}

We model each round of interaction as a cooperative game $G = (N,v)$, where the players $i$ are the data sellers $N = \{1, \ldots, n\}$, each holding a dataset $D_i$. A natural characteristic function is given by the accuracy $\mathcal{G}(\model, D_{\validation})$ achieved by the model $\model$ trained on the data held by players in the coalition:
\begin{equation*}
  v(\S) \coloneqq \mathcal{G}(\model(\cup_{i\in \S} D_i), D_{\validation})
\end{equation*}

Submodularity is often a good model for approximating properties of this accuracy---the value of additional training datasets typically diminishes with growing data size~\cite{kirchhoff2014submodularity}. 

\subsection{Data Replication Attack and Replication Robustness} 
In the context of a data market game, a replication manipulation can be implemented by a malicious player through replicating its data and acting under multiple false identities. 
For many ML models, redundant data do not significantly change the model's performance and hence satisfies the replication redundancy assumption. As the market game is an instance of submodular game with replication redundant characteristic function, we can directly apply our replication robustness results shown for submodular games to evaluate the common solution concepts. That is, in a market game $G$, the Shapley value is not robust against the data replication attack, and the malicious player is incentivised to replicate its data and act under multiple identities in order to increase its total payoff. Whereas the Banzhaf value, Leave-one-out value, and Robust Shapley value are robust against the data replication attack.
Similarly, the conditions for replication robustness presented in Section~\ref{subsec:robustness_conditions} also hold for the ML data market game.

%% file: sections/experiments.tex
In this section, we provide empirical validations to our theoretical results on the ML data market. 
Specifically, we will present experiments which justify our assumptions on submodularity and replication redundancy. We then compare the replication robustness of the discussed solution concepts. 

\subsubsection{Experiment Setup}
We test our results on three standard ML tasks (datasets) of varied sizes. On each task, we assign to each player a subset of the data, and a malicious player replicates its data and we gradually increase the number of replications $k$. The datasets and assignments are as follows:
\begin{enumerate}[label=(\alph*)=, leftmargin=*]
    \itemsep 0mm
     \item \emph{Covertype}~\cite{Dua:2019}: Each input consists of 10 continuous features (e.g., elevation, slope, hillshade 9am, etc.), and the output is a prediction of the forest cover type out of 7 classes. We use the dataset provided by Kaggle which consists of $\sim$15000 training datapoints uniformly distributed in the $7$ output classes. 5 honest players each holds 1000 datapoints, 5 replicas share 1000 datapoints.
    \item \emph{CIFAR-100}~\cite{krizhevsky2009learning}: 32x32x3 images of 20 superclasses and 100 subclasses $C_\textnormal{sub}$. 
    We carried out $4$ sets of experiments with varied data assignments as follows:
    \begin{itemize}
        \item \emph{Uniform: } Players $0-4$ assigned data from 100 $C_\textnormal{sub}$ uniformly, players $5-7$ (replicas) same as Player $0$. 
        \item \emph{Disjoint: } Players $0-4$ each assigned 20 $C_\textnormal{sub}$, players $5-7$ (replicas) assigned the same data as Player $0$.
        \item \emph{Mixed: }Players $0-4$ assigned varied portions of each $C_\textnormal{sub}$,  players $5-7$ (replicas) same as Player $0$.
    \end{itemize}
    \item \emph{Tiny ImageNet}~\cite{tiny}: 64x64x3 images of 20 random classes. 3 honest players each holds 2000 datapoints and 3 replicas hold the same 2000 datapoints.
\end{enumerate}


To construct the ML models, we used a 4-layer (512 units per layer) fully-connected neural network for Covertype prediction. For CIFAR-100, we used the VGG-16 architecture \cite{simonyan2014vgg} with 10 convolutional layers (kernel size 3), max-pooling, and 2 fully-connected layers (1024 units per layer). For Tiny-ImageNet, we used the VGG-16 with 2 fully-connected layers (4096 units per layer).
Adam optimizer~\cite{kingma2014adam} is used to train the models. For the Covertype classification, we use learning rate of 0.0001, minibatch size 128. For CIFAR-100, we use learning rate of 0.001, minibatch size 64. For Tiny ImageNet we use learning rate of 0.001, minibatch size 64.



\subsubsection{Validations on properties of the ML Data Market}
We empirically validate Assumption~\ref{assumption:submodularity} (submodularity) and Assumption~\ref{assumption:replication} (replication redundancy).
Fig.~\ref{fig:zc} shows the average marginal contributions $z_i(c)$ for each player over coalition sizes $c$. Observe that $z_i(c)$ is monotonic decreasing, which according to Lemma~\ref{lemma:zc_monotone}, is a result of the submodularity of the characteristic function.
The curves further validate replication redundancy with $z_i(c) \approx 0$ for the replica players when $c$ exceeds the number of honest players. 

\begin{figure*}[h]
	\centering
	\subfloat[Covertype]{\includegraphics[width=0.24\linewidth]{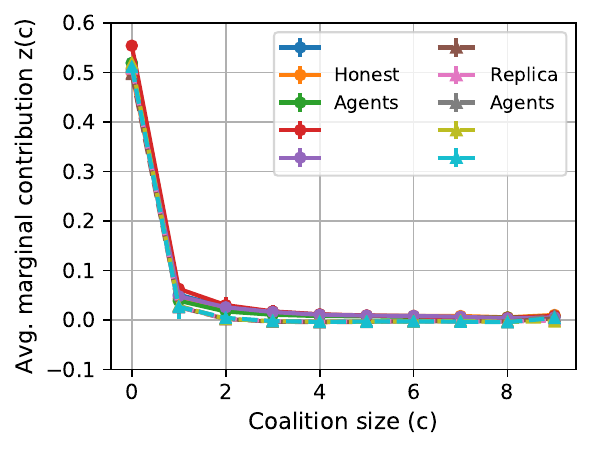}}%
	\hfill
	\subfloat[CIFAR-100 (uniform)]{\includegraphics[width=0.24\linewidth]{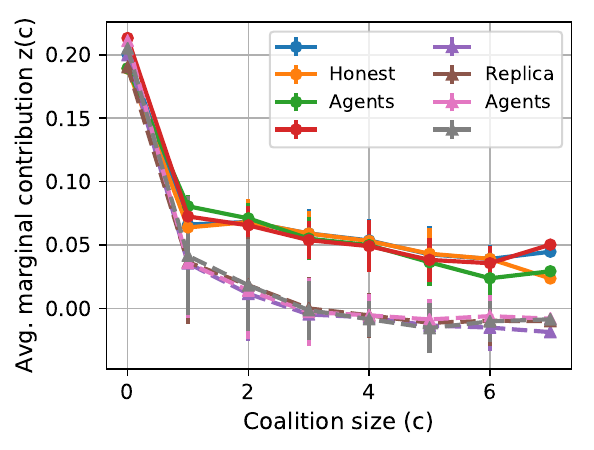}}%
	\hfill
	\subfloat[CIFAR-100 (disjoint)]{\includegraphics[width=0.25\linewidth]{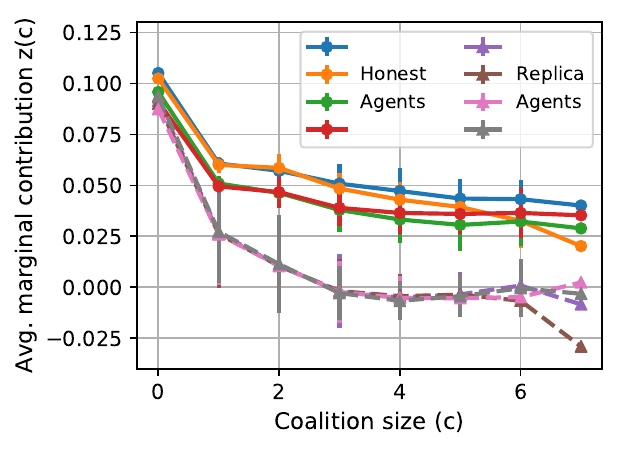}}%
	\hfill
	\subfloat[Tiny ImageNet]{\includegraphics[width=0.24\linewidth]{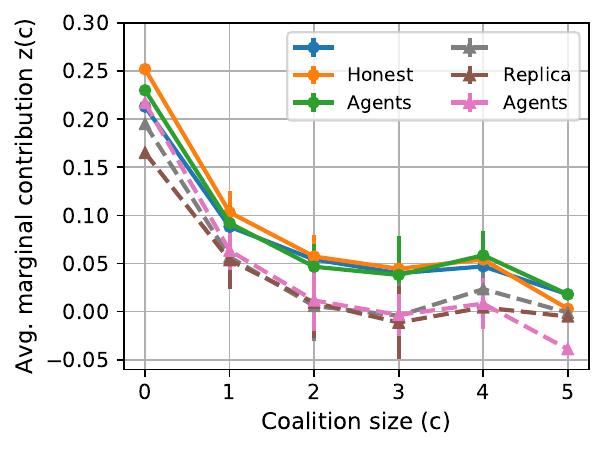}}%
	\hfill
	\caption{Average marginal contributions $z_i(c)$ across various datasets. Solid lines are non-replicating players while dashed lines are replicas which belong to the malicious player. Error bars show standard deviations of the marginal contributions of each coalition size. Observe that $z_i(c)$ monotonic decreases with coalition size as a result of submodularity. }
	\label{fig:zc}
\end{figure*}

\begin{figure*}
	\centering
	\begin{minipage}{1\textwidth}\centering
	\includegraphics[height=0.03\textwidth]{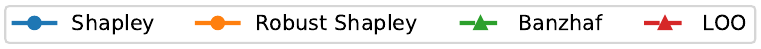}
	\end{minipage}
	\subfloat[Covertype]{\includegraphics[width=0.2\linewidth]{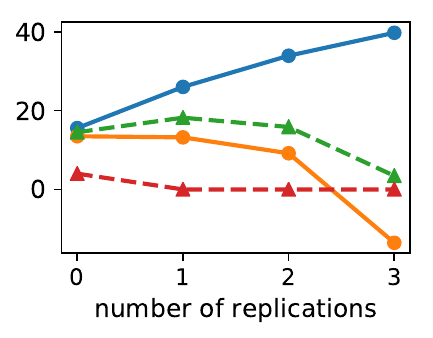}\label{fig:robustness_a}}%
	\hfill
	\subfloat[CIFAR(uniform)]{\includegraphics[width=0.2\linewidth]{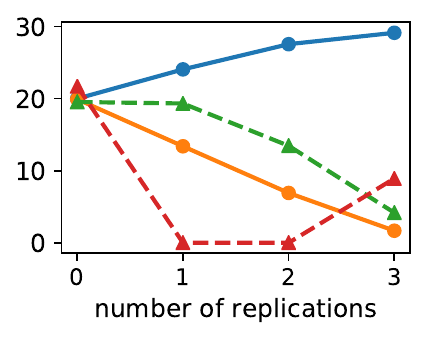}\label{fig:robustness_d}}%
	\hfill
	\subfloat[CIFAR (disjoint)]{\includegraphics[width=0.2\linewidth]{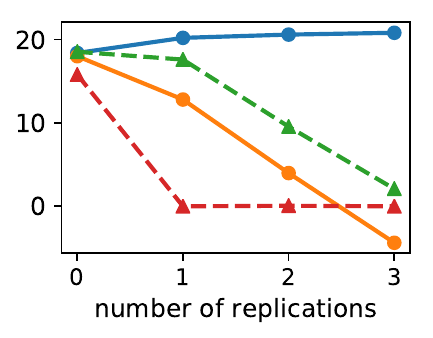}\label{fig:robustness_b}}%
	\hfill
	\subfloat[CIFAR (mixed)]{\includegraphics[width=0.2\linewidth]{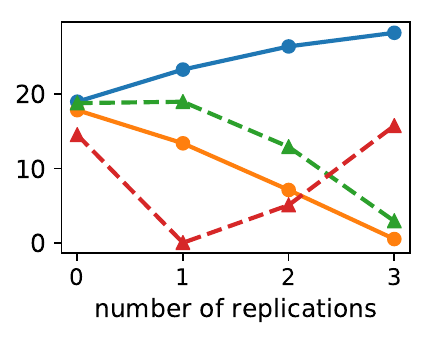}\label{fig:robustness_c}}%
	\subfloat[Tiny ImageNet]{\includegraphics[width=0.2\linewidth]{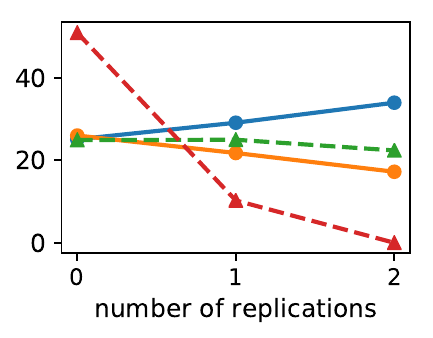}\label{fig:robustness_e}}%
	\hfill
	\caption{Percentage of total replica values in the total allocated payoffs w.r.t number of replications. x-axis represents increasing number of replications by the malicious player, e.g. x=3 refers to an induced game where the malicious player holds 4 replicas.}
	\label{fig:replication_robustness}
	\vspace{-5mm}
\end{figure*}

\subsubsection{Replication Robustness}
Fig.~\ref{fig:replication_robustness} compares the replication robustness of various solution concepts. The curves show the changes in total payoffs of the malicious player as a percentage of the total allocated payoffs, over growing number of replicas. On the Covertype, CIFAR-100, and Tiny ImageNet tasks, we start with 5,4,3 honest players respectively and 1 malicious player, and along the x-axis, we gradually increase the number of replicas. In all settings, the Shapley value is vulnerable to replication, and the total share of value gained by the replica player increases. Both the Banzhaf value and Robust Shapley value are replication robust. The Leave-one-out value is sensitive to the randomness during training, because it only includes a player's marginal contribution towards all other players. We plot the percentage for easy comparison, which also preserves the trend of the actual value.

%% file: sections/related.tex
Our theoretical results relate closely to the seminal game-theoretic literature on merging/splitting proofness, collusion, and false name manipulations.
\citet{lehrer1988axiomatization} is the first to present an axiomatization of the Banzhaf value with the 2-efficiency axiom, which characterized the neutrality of the Banzhaf value on merging two players as one. Similarly,~\citet{haller1994collusion} studied the collusion of two players where they both keep their identities thus the total number of players are unchanged: under a proxy agreement, one player acts as a proxy while the other a null player, whereas under an association agreement, two players act on each other's behalf. \citet{van2012efficiency} studied the interplay between efficiency and collusion neutrality of two players. 
\citet{knudsen2012merging} studied the merging and splitting-proofness on convex games, and introduced some possibility/impossibility results. 
\citet{ohta2008anonymity} studied false name manipulation in an open environment and proposed anonymity-proof Shapley value against malicious players who split their \emph{skills} and act as multiple identities, where skills are assumed to be unique.
Related to the splitting manipulations, our present work looks at the replication manipulation arising in submodular games. Such cases have not been adequately addressed previously. Moreover, our results extends from bilateral amalgamation to an arbitrary number of replica players. 
Related work on emerging ML applications include (1) ML data markets, e.g., \citet{agarwal2019marketplace} first introduced an algorithmic framework for data marketplaces.~\citet{ohrimenko2019collaborative} studied collaborative ML data markets where each player must participate both as seller and buyer. (2) ML model interpretation~\cite{vstrumbelj2014explaining, janzing2020feature}, which explains ML models through the feature importance. Many have adopted game-theoretic solution concepts such as the Shapley Value~\cite{lundberg2017unified, sundararajan2020many, cohen2005feature,chen2018shapley}, and (3) Submodular data and feature selection~\cite{wei2015submodularity, kirchhoff2014submodularity,das2012selecting,lin2011class}.

%% file: sections/conclusions.tex
In this work, we studied the robustness of solution concepts against redundancy as a result of replication in submodular games. In summary, we showed a necessary and sufficient condition which characterises the robustness of semivalues in general. Using this condition, we showed that the Shapley value is not replication robust, i.e., the total payoff of the malicious player monotonic increases with growing number of replications. Whereas the Banzhaf value, Robust Shapley and Leave-one-out value are replication robust. We demonstrate the distinct robustness and convergence properties of the Shapley and Banzhaf value on a submodular facility location game. Moreover, we applied our theoretical results to an emerging application of ML data markets, and empirically validated our theoretical results across three standard ML datasets. Interesting future directions include extending our theoretical framework for submodular games with partial redundancy; and applying our theoretical findings to submodular ML applications such as feature evaluation and multiagent learning.

%% file: sections/appendix.tex

\section{Proof for Theorem~\ref{thm:shapleysinglerobustness}}\label{appendix:shapleysinglepayoff}
\shapleysinglerobustness*

\begin{proof}
The second half of the theorem is provided in the main text, here we provide the derivation for $\delta^\textnormal{Shapley}_i$. In the induced game $G^R = (N^R, v^R)$ where player $i$ replicates into two players $\{i_1,i_2\}$, the total number of players increases by one, i.e., $|N^R| = |N| + 1$ and $v^R(\S\cup \{i_1,i_2\}) = v^R(\S\cup\{i\})$ as $i_1$ and $i_2$ are replicas of $i$ and as a result of replication redundancy. We next write out the sum of the Shapley values $\val^R_{i_2}$ and $\val^R_{i_2}$ of $i_1,i_2$ in $G^R$.
\begin{align*}
\val^R_{i_1} (N^R, \cv^R)
	& \coloneqq \sum_{\S\subseteq N^R \setminus \{i\}} \frac{|\S|!(|N^R|-|\S| -1)!}{|N^R|!} \mc_{i_1}(\S)\\
	& = \sum_{\S\subseteq N^R \setminus \{i\}} \frac{|\S|!(|N|+1-|\S| -1)!}{(|N|+1)!} \mc_{i_1}(\S)\\
	& \stackrel{(1)}{=} \sum_{\S\subseteq N^R \setminus \{i_1,i_2\}} \frac{|\S|!(|N|-|\S|)!}{(|N|+1)!} \mc_{i_1}(\S)+  \sum_{\S\subseteq N^R \setminus \{i_1,i_2\}} \frac{(|\S|+1)!(|N|-|\S|-1)!}{(|N|+1)!} \!\!\!\!\underbrace{\mc_{i_1}(\S\cup\{i_2\})}_{\textnormal{$= 0$, replication redundancy}}\\
	& = \underbrace{\sum_{\S\subseteq N^R \setminus \{i_1,i_2\}}}_{ = N \setminus \{i\}} \frac{|\S|!(|N|-|\S|)!}{(|N|+1)!} \underbrace{\mc_{i_1}(\S)}_{=\mc_i(\S)}
	= \sum_{\S\subseteq N \setminus \{i\}} \frac{|\S|!(|N|-|\S|)!}{(|N|+1)!} \mc_{i}(\S),
\end{align*} 
where (1) is by grouping the coalitions of players (excluding $i_1$) into two groups, one group containing all coalitions without $i_2$ and the other with $i_2$ in.

By symmetry, $\val^R_{i_2} (N^R, \cv^R) = \val^R_{i_1} (N^R, \cv^R) $, and the total payoff of $i$ in the induced game is:
\begin{equation*}
	\val^R_{i} = 2 \val^R_{i_1} (N^R, \cv^R) =\sum_{\S\subseteq N \setminus \{i\}} \frac{2|\S|!(|N|-|\S|)!}{(|N|+1)!} \mc_{i}(\S)
\end{equation*}
On the other hand, the total payoff of player $i$ in the original game is:
\begin{equation*}
\val_i = \sum_{\S\subseteq N \setminus \{i\}} \frac{|\S|!(|N|-|\S| -1)!}{|N|!} \mc_{i}(\S)
\end{equation*}
Therefore, the change in total payoff of player $i$ is :
\begin{align*}
\delta\vshapley_i 
&=\val^R_i - \val_i\\
&= \sum_{\S\subseteq N \setminus \{i\}} \frac{|\S|!(|N|-|\S| -1)!}{|N|+1!}(2(|N|-|\S|)-(|N|+1))\mc_i(\S)\\
&=  \sum_{\S\subseteq N \setminus \{i\}} \frac{|\S|!(|N|-|\S| -1)!}{|N|+1!}(|N|-2|\S|-1) \mc_i(\S)
\end{align*}
\end{proof}


\section{Proof for Theorem~\ref{thm:banzhafsinglepayoff}}\label{appendix:banzhafsinglepayoff}

\banzhafsinglepayoff*

\begin{proof}
In the induced game $G^R = (N^R, v^R)$ where player $i$ replicates into two players $\{i_1,i_2\}$, the total number of players increases by one, i.e., $|N^R| = |N| + 1$ and $v^R(\S\cup \{i_1,i_2\}) = v^R(\S\cup\{i\})$ as $i_1$ and $i_2$ are replicas of $i$ and as a result of replication redundancy. We next write out the sum of the Banzhaf values $\val^R_{i_2}$ and $\val^R_{i_2}$ of $i_1,i_2$ in $G^R$.
\begin{align*}
\val^R_{i_1} (N^R, \cv^R)
	& \coloneqq \sum_{\S\subseteq N^R \setminus \{i\}} \frac{1}{2^{|N^R|}} \mc_{i_1}(\S)\\
	& = \sum_{\S\subseteq N^R \setminus \{i\}} \frac{1}{2^{|N|+1}} \mc_{i_1}(\S)\\
	& \stackrel{(1)}{=} \sum_{\S\subseteq N^R \setminus \{i_1,i_2\}} \frac{1}{2^{|N|+1}} \mc_{i_1}(\S)+  \sum_{\S\subseteq N^R \setminus \{i_1,i_2\}} \frac{1}{2^{|N|+1}} \underbrace{\mc_{i_1}(\S\cup\{i_2\})}_{\textnormal{$= 0$, replication redundancy}}\\
	& = \underbrace{\sum_{\S\subseteq N^R \setminus \{i_1,i_2\}}}_{ = N \setminus \{i\}} \frac{1}{2^{|N|+1}} \underbrace{\mc_{i_1}(\S)}_{=\mc_i(\S)}
	= \sum_{\S\subseteq N \setminus \{i\}} \frac{1}{2^{|N|+1}} \mc_{i}(\S),
\end{align*} 
where (1) is by grouping the coalitions of players (excluding $i_1$) into two groups, one group containing all coalitions without $i_2$ and the other with $i_2$ in.

By symmetry, $\val^R_{i_2} (N^R, \cv^R) = \val^R_{i_1} (N^R, \cv^R) $, and the total payoff of $i$ in the induced game is:
\begin{equation*}
	\val^R_{i} = 2 \val^R_{i_1} (N^R, \cv^R) =\sum_{\S\subseteq N \setminus \{i\}} \frac{1}{2^{|N|}} \mc_{i}(\S) = \val_i(N, \cv)
\end{equation*}
Therefore, the change in total payoff of player $i$ is zero, i.e.,:
\begin{align*}
\delta\vbanzhaf_i =\val^R_i - \val_i = 0
\end{align*}
\end{proof}


\section{Proof for Equation~\eqref{eq:zc}}\label{appendix:zc}

\begin{proof}
The derivation from Equation~\eqref{eq:ungrouped_semivalue} to \eqref{eq:zc} can be obtained by grouping the marginal contributions of player $i$ towards equal-sized coalitions. The payoff of the player $\val_i$ can be computed as a weighted sum of its average marginal contributions $z_i(c)$ to each coalition size $c$,  where $\tbinom{|N|-1}{c}$ is the number of size-$c$ coalitions of players excluding $i$. The detailed steps are as follows:
\begin{align*}
        \varphi_i(N, v)
            &= \sum_{\S\subseteq N\setminus \{i\}} \weight{|\S|}{N} \mc_i(\S)\\
            &= \sum_{c=0}^{|N|-1} \sum_{|\S| = c, \S\subseteq N\setminus \{i\}} \weight{|\S|}{N}\mc_i(\S)
            = \sum_{c=0}^{|N|-1}\weight{|\S|}{N} \sum_{|\S| = c, \S\subseteq N\setminus \{i\}}\mc_i(\S)\\
            &= \sum_{c=0}^{|N|-1}\underbrace{\tbinom{|N|-1}{c}\weight{|\S|}{N}}_{\alpha_c} \underbrace{{\tbinom{|N|-1}{c}}^{-1}\sum\nolimits_{|\S| = c, \S\subseteq N\setminus \{i\}}\mc_i(\S)}_{z_i(c)}
            = \sum_{c=0}^{|N|-1}\alpha_{c} z_i(c).\qedhere
    \end{align*}
\end{proof}


\section{Proof for Lemma~\ref{lemma:ack}}\label{appendix:payoff_changes}

\lemmaack*
\begin{proof}
In the induced game $G^R$, let $\weight{|\S|}{N+k}$ be the weights of the solution concept by definition, i.e., $\varphi_i(N^R, \cv^R) \coloneqq \sum_{\S\subseteq N^R\setminus\{i\}}\weight{|\S|}{N+k}\mc_i(\S)$. Then the total payoff of the malicious player after $k$ replications is:

\begin{align*}
    \varphi_i^\textnormal{tot}(k) 
    &\stackrel{(1)}{=} (k+1)\varphi_{i_k}(N^R, v^R)\\
    &= (k+1)\sum_{\S\subseteq N^R\setminus\{i_k\}}\weight{|\S|}{|N|+k}  \mc_{i_k}(\S)\\
    & = (k+1) \sum_{\S\subseteq N^R \setminus \S_R} \weight{|\S|}{|N|+k}\mc_{i_k}(\S) + (k+1)\sum_{\S\subseteq N^R \setminus \{i_k\}, \S\cap \S_R\neq \phi } \weight{|\S|}{|N|+k}\underbrace{\mc_{i_k}(\S)}_{\stackrel{(2)}{=}0} \\
    &= (k+1)\sum_{S\subseteq{N\setminus\{i\}}} \weight{|\S|}{|N|+k} \mc_i(\S)\\
    &=\sum_{c=0}^{N-1}\underbrace{(k+1)\tbinom{N-1}{c}\weight{|\S|}{|N|+k}}_{\alpha_c^k} z_i(c)
\end{align*}
where (1) is due to symmetry and (2) is due to replication-redundancy.
\end{proof}


\section{Proof for Theorem~\ref{thm:gain_increase} for the Banzhaf value and LOO}\label{appendix:gain_increase}

\shapleysubmodular*
\begin{proof}
We have proven the robustness properties for the Shapley value in the main text. Here we will provide the proofs for the Banzhaf and LOO values.

\textbf{(1) For the Banzhaf value. }
    We prove that the importance weights $\alpha_c^k$ satisfy $\forall k\geq 0, \alpha_c^k \geq \alpha_c^{k+1}$ for both the Banzhaf value and Leave-one-out, which is a sufficient condition for $\forall p, \sum_{c=0}^p \alpha_c^0 \geq \sum_{c=0}^p \alpha_c^k$. Therefore, according to our robustness condition in Theorem~\ref{thm:robustness_condition}, both the Banzhaf value and Leave-one-out value are replication robust. In addition, both values monotonic decrease as the number of replications $k$ according to Corollary \ref{thm:monotonicity_condition}. In particular, for the Banzhaf value, \begin{equation*}
    \begin{cases}\alpha_c^k = \frac{(k+1)}{2^{|N|+k-1}}\tbinom{|N|-1}{c}\\
        \alpha_c^{k+1} = \frac{(k+2)}{2^{|N|+k}}\tbinom{|N|-1}{c}
        \end{cases} \implies \frac{\alpha_c^k}{\alpha_c^{k+1}} = 2 \tfrac{(k+1)}{(k+2)} \geq 1
\end{equation*}
   
As we can see, $\alpha_c^0 = \alpha_c^1$ for the Banzhaf value. This implies that the total payoff of the malicious player is unchanged when it replicates for the first time. Therefore, the Banzhaf value is neutral for $k=1$, which conforms with our finding in Theorem~\ref{thm:banzhafsinglepayoff}.
The limit of the total payoff is:

\begin{equation*}
\lim_{k\rightarrow{\infty}}\varphi_i^\textnormal{tot}(k) 
    = \lim_{k\rightarrow{\infty}}\sum_{c=0}^{|N|-1}\alpha_c^k z_i(c)
    = \sum_{c=0}^{|N|-1}\tbinom{|N|-1}{c}z_i(c)\lim_{k\rightarrow{\infty}}\frac{k+1}{2^{|N|+k-1}} = 0
\end{equation*}

\noindent
\textbf{(2) For the Leave-one-out value.}
$\forall k\geq 0, \frac{\alpha_c^{k+1}}{\alpha_c^{k}} = 0$, hence $\frac{\alpha_c^k}{\alpha_c^{k+1}} \geq 1$, and
the total payoff is zero with any positive number of replications, i.e., $\forall {k>0}, \varphi_i^\textnormal{tot}(k) = 0$.
\end{proof}

\section{Proofs for Lemma~\ref{lemma:sum_alpha_properties}}\label{appendix:shapley_replication}
\shapleyproperties*

\begin{proof} 
\textbf{Proof for Equation~(\ref{subeq:efficiency}):}
Equation~\eqref{subeq:efficiency} shows that $\alpha_c^k$ always sums to 1 under changing $k$ for the Shapley value.
\begin{align*}
    \sum_{c=0}^{|N|-1} \alpha_c^k
        &= \sum_{c=0}^{|N|-1} \frac{k+1}{|N|+k}\binom{|N|-1}{c}{\binom{|N|+k-1}{c}}^{-1} \textnormal{ , due to Corollary~\ref{lemma:example_alpha_ck}}\\
        &=(k+1)\sum_{c=0}^{|N|-1} \frac{(|N|-1)!(|N|+k-1-c)!}{(|N|-1-c)!(|N|+k)!}\\
        &= \frac{(k+1)!(|N|-1)!}{(|N|+k)!}\sum_{c=0}^{|N|-1} \frac{(|N|+k-1-c)!}{(|N|-1-c)!k!}\\
        &= \frac{1}{\binom{|N|+k}{k+1}}\sum_{c=0}^{|N|-1} \binom{|N|+k-1-c}{k}\\
        \overset{(1)}&{=} \frac{1}{\binom{|N|+k}{k+1}}\sum_{i=k}^{|N|-k-1} \binom{i}{k} \\
        \overset{(2)}&{=} \frac{1}{\binom{|N|+k}{k+1}}\binom{|N|+k}{k+1}\\
        &= 1, 
\end{align*}
where (1) is by substituting $i=|N|+k-1-c$ and (2) by the Hockey-Stick identity.\\

\noindent
\textbf{Proof for Equation~(\ref{subeq:monotonicity}):} This shows that the importance weights $\alpha_c^k$ \emph{shift} to the smaller coalitions under growing $k$. 

\begin{align*}
    \sum_{c=0}^p \alpha_c^k
        &=\sum_{c=0}^p \frac{(k+1)(|N|-1)!(|N|+k-1-c)!}{(|N|-1-c)!(|N|+k)!}\\
        &= \frac{(k+1)!(|N|-1)!}{(|N|+k)!}\sum_{c=0}^p \frac{(|N|+k-1-c)!}{(|N|-1-c)!k!}\\
        &= \frac{1}{\binom{|N|+k}{k+1}}\sum_{c=0}^p \binom{|N|+k-1-c}{k}\\
        &= \frac{1}{\binom{|N|+k}{k+1}}(\sum_{c=0}^{|N|-1} -\sum_{c=p+1}^{|N|-1} \binom{|N|+k-1-c}{k})\\
        \overset{(1)}&{=} 1 - \frac{1}{\binom{|N|+k}{k+1}}\sum_{c=p+1}^{|N|-1} \binom{|N|+k-1-c}{k}\\
       \overset{(2)}&{=} 1 - \frac{1}{\binom{|N|+k}{k+1}}\binom{|N|+k-p-1}{k+1}\\
        &= 1- \frac{(|N|-1)!}{(|N|-p-2)!}\frac{1}{(|N|+k)...(|N|+k-p)},
\end{align*}

where (1) is by Equation~\eqref{subeq:efficiency} and (2) is by the Hockey-Stick identity.
Similarly,

\begin{equation*}
   \sum_{c=0}^p \alpha_c^{k+1}
        = 1- \frac{(|N|-1)!}{(|N|-p-2)!}\frac{1}{(|N|+k+1)...(|N|+k+1-p)}
\end{equation*}

Therefore,

\begin{align*}
     \sum_{c=0}^p \alpha_c^{k+1} - \sum_{c=0}^p \alpha_c^{k}
        &=\frac{(|N|-1)!}{(|N|-p-2)!}\frac{(|N|+k+1) - (|N|+k-p)}{(|N|+k+1)...(|N|+k-p)}\\
        &= \frac{(|N|-1)!}{(|N|-p-2)!}\frac{p+1}{(|N|+k+1)...(|N|+k-p)}
       \geq 0
\end{align*}

\textbf{Proof for Equation~(\ref{subeq:submodularity}):} With this additional condition, the \emph{unit gain} of the total payoff for adding a replica decreases monotonically for each replication. This means that the player obtains the most unit gain for the first replication.
To show this property, we denote for any $k$, denote $\delta^k \coloneqq \sum_{c=0}^p \alpha_c^{k+1} - \sum_{c=0}^p \alpha_c^{k}$. From the proof of Equation~\eqref{subeq:monotonicity}: $\delta^k = \tfrac{(|N|-1)!}{(|N|-p-2)!}\tfrac{p+1}{(|N|+k+1)...(|N|+k-p)}$, therefore,

\begin{align*}
RHS - LHS \textnormal{ of~\eqref{subeq:submodularity}} 
&= \delta^{k+1} - \delta^{k}\\
&=\frac{(|N|-1)!(p+1)}{(|N|-p-2)!}(\frac{1}{(|N|+k+2)...(|N|+k+1-p)} - \frac{1}{(|N|+k+1)...(|N|+k-p)})\\
&= \frac{(|N|-1)!(p+1)}{(|N|-p-2)!}(\frac{(|N|+k-p) - (|N|+k+2)}{(|N|+k+2)...(|N|+k-p)})\\
&=\frac{(|N|-1)!(p+1)}{(|N|-p-2)!}\frac{-(p+2)}{(|N|+k+2)...(|N|+k-p)} \leq 0 \qedhere
\end{align*}
\end{proof}


\section{Proof for Corollary~\ref{corollary:robust_shapley_loss}}\label{appendix:robust_shapley}
\corollaryloss*
\begin{proof}
\textbf{Replication-robustness}
We prove that similar to the Banzhaf value, the Robust Shapley value satisfies Equation~\eqref{eq:robustness_sufficient_condition_summand} in Observation 1: $\forall k\geq0, \frac{\alpha_c^k}{\alpha_c^{k+1}} \geq 1$. Hence it satisfies Theorem \ref{thm:monotonicity_condition}, and therefore sufficient for replication robustness. There are 3 possible cases:

\noindent
\textit{Case 1: $c < \lfloor \frac{|N|+k-1}{2}\rfloor \leq \lfloor\frac{|N|+k}{2}\rfloor$}

In this case, both $\Tilde{\alpha}_c^k$ and $\Tilde{\alpha}_c^{k+1}$ will be down-weighed from the Shapley coefficients where $\gamma_{|N|+k}^c = \frac{\lceil \frac{|N|+k-1}{2}\rceil!\lfloor \frac{|N|+k-1}{2}\rfloor!}{c! (|N|+k-c-1)!}$:
{\begin{align*}
    \Tilde{\alpha}_c^k &= \gamma_{|N|+k}^{c}\alpha_c^k =(k+1)\tbinom{|N|-1}{c}\frac{ \lfloor \frac{|N|+k-1}{2}\rfloor! \lceil \frac{|N|+k-1}{2}\rceil!}{(|N|+k)!}\phantom{space}\\
    \Tilde{\alpha}_c^{k+1} &= \gamma_{|N|+k+1}^{c}\alpha_c^{k+1} =(k+2)\tbinom{|N|-1}{c}\frac{ \lfloor \frac{|N|+k}{2}\rfloor! \lceil \frac{|N|+k}{2}\rceil!}{(|N|+k+1)!}\\
    \textnormal{Hence }&\frac{\Tilde{\alpha}_c^{k}}{\Tilde{\alpha}_c^{k+1}} = \frac{k+1}{k+2}\frac{|N|+k+1}{\lceil \frac{|N|+k}{2} \rceil} \geq \frac{1}{2}*2 = 1.
\end{align*}}
\noindent
\textit{Case 2: $c \geq \lfloor\frac{|N|+k}{2}\rfloor \geq \lfloor \frac{|N|+k-1}{2}\rfloor$}

Both $\Tilde{\alpha}_c^k$, $\Tilde{\alpha}_c^{k+1}$ take the original form of Shapley coefficients after replication, i.e, $\gamma_{|N|}^c = 1$:
\begin{align*}
    \Tilde{\alpha}_c^k &= \alpha_c^k =(k+1)\tbinom{|N|-1}{c}\frac{c! (|N|+k-1-c)!}{(|N|+k)!}\phantom{whitespace}\\
    \Tilde{\alpha}_c^{k+1} &=\alpha_c^{k+1} =(k+2)\tbinom{|N|-1}{c}\frac{c! (|N|+k-c)!}{(|N|+k+1)!}\\
    \frac{\Tilde{\alpha}_c^{k}}{\Tilde{\alpha}_c^{k+1}} &= \frac{k+1}{k+2}\frac{|N|+k+1}{|N|+k-c} \stackrel{(1)}{\geq} 2\frac{k+1}{k+2} \geq 1, \quad\textnormal{ where (1) is due to } c \geq \lfloor\tfrac{|N|+k}{2}\rfloor.
\end{align*}
\noindent
\textit{Case 3: $\lfloor\frac{|N|+k-1}{2}\rfloor \leq c < \lfloor\frac{|N|+k}{2}\rfloor$}

In this case, $\Tilde{\alpha}_c^k$ will take the original form, while $\Tilde{\alpha}_c^{k+1}$ will take the down-weighed form. Moreover, $|N|+k$ must be even, hence $c = \lfloor\frac{|N|+k-1}{2}\rfloor$.
{
\begin{align*}
    \Tilde{\alpha}_c^k &= \alpha_c^k =(k+1)\tbinom{|N|-1}{c}\frac{c! (|N|+k-1-c)!}{(|N|+k)!}
    =(k+1)\tbinom{|N|-1}{c}\frac{ \lfloor \frac{|N|+k-1}{2}\rfloor! \lceil \frac{|N|+k-1}{2}\rceil!}{(|N|+k)!} \\
    \Tilde{\alpha}_c^{k+1} &= \gamma_{|N|+k+1}^{c}\alpha_c^{k+1} =(k+2)\tbinom{|N|-1}{c}\frac{ \lfloor \frac{|N|+k}{2}\rfloor! \lceil \frac{|N|+k}{2}\rceil!}{(|N|+k+1)!}\\
    \textnormal{Hence }& \frac{\Tilde{\alpha}_c^{k}}{\Tilde{\alpha}_c^{k+1}} = \frac{k+1}{k+2}\frac{|N|+k+1}{\lceil \frac{|N|+k}{2} \rceil} \geq 2 \frac{k+1}{k+2} \geq 1
\end{align*}}
We have shown that $\forall k\geq0, \frac{\alpha_c^k}{\alpha_c^{k+1}} \geq 1$, and hence the Robust Shapley value is replication-robust.

\noindent
\textbf{Payoff loss}
Note that from the above derivations, in all 3 cases, $\forall k\geq 0, \frac{k+2}{k+1}\frac{\Tilde{\alpha}_c^k}{\Tilde{\alpha}_c^{k+1}} \geq 2$:

\begin{align*}
    \!\!\!\!\varphi_i^\textnormal{tot}(0) &= \sum_{c=0}^{|N|-1}\Tilde{\alpha}_c^0 z_i(c) \coloneqq \dfrac{1}{|N|}\sum_{c=0}^{|N|-1}\gamma_|N|^c z_i(c)\\
    \varphi_i^\textnormal{tot}(k) &=\sum_{c=0}^{|N|-1}\Tilde{\alpha}_c^k z_i(c)\\
    &= (k+1)\sum_{c=0}^{|N|-1}\frac{\Tilde{\alpha}_c^k}{k+1} z_i(c), \textnormal{ and as $\forall {k\geq0}, $} \tfrac{\Tilde{\alpha}_c^k/(k+1)}{\Tilde{\alpha}_c^{k+1}/(k+2)} \geq 2, \\
    &\leq (k+1)\sum_{c=0}^{|N|-1}\frac{1}{2}\frac{\Tilde{\alpha}_c^{k-1}}{k} z_i(c) \leq ... \\
    &\leq (k+1)\sum_{c=0}^{|N|-1}\frac{1}{2^k}\Tilde{\alpha}_c^{0} z_i(c)\\
    &= (\dfrac{k+1}{2^k})\dfrac{1}{|N|}\sum_{c=0}^{|N|-1}\gamma_{|N|}^c z_i(c)\\
    \textit{Hence } &\varphi^\textnormal{tot}_i(0) - \varphi^\textnormal{tot}_i(k)\geq \dfrac{1}{|N|}\sum_{c=0}^{|N|-1}(1-\tfrac{k+1}{2^k})\gamma_{|N|}^c z_i(c).
\end{align*}
This concludes our proof for Corollary~\ref{corollary:robust_shapley_loss}.
\end{proof}

\section{Proofs for Lemma~\ref{lemma:perturbation}}\label{appendix:attack}

\perturbation*
\begin{proof}
 Compared with replication, the additional gain in payoff due to the perturbation is 
\begin{align*}
    \varphi_i^\textnormal{perturb}- \varphi_i^\textnormal{replicate} = 
     &\sum_{p_k \in \S^P}\sum_{\S\subseteq N\setminus\{i\}, \S^p\subseteq \S^P\setminus\{p_k\}}\weight{|\S\cup \S^p|}{|N|+k} \mc_{p_k}(\S\cup \S^p) - \\ &\phantom{some more space}\sum_{i_k \in \S^R}\sum_{\S\subseteq N\setminus\{i\}, \S^r\subseteq \S^R\setminus\{i_k\}}\weight{|\S\cup \S^r|}{|N|+k} \mc_{i_k}(\S\cup \S^r)\\
     &= \sum_{k=0}^{|\S^R|-1}\sum_{\S\subseteq N\setminus\{i\}} \weight{|\S|}{|N|+k}(\mc_{p_k}(\S) - \mc_{i_k}(\S)) + \\ &\phantom{some more space} \sum_{p_k\in \S^P}\sum_{\S\subseteq N\setminus\{i\}, \S^p\subseteq_{\neq \emptyset} \S^P\setminus\{p_k\}} \weight{|\S\cup \S^p|}{|N|+k}\mc_{p_k}(\S \cup\S^p) \\
     &= \sum_{p_k\in \S^P}\sum_{\S\subseteq N\setminus\{i\}, \S^p\subseteq_{\neq \emptyset} \S^P\setminus\{p_k\}} \weight{|\S\cup \S^p|}{|N|+k}\mc_{p_k}(\S \cup \S^p)\\
     &\stackrel{(1)}{\leq} (k+1)\sum_{\S\subseteq N\setminus\{i\}, \S^p\subseteq_{\neq \emptyset} \S^P\setminus\{p_k\}}\weight{|\S\cup \S^p|}{|N|+k}\epsilon\\
     &\stackrel{(2)}{\leq}(k+1)\epsilon,
\end{align*}
where (1) is due to the assumption on $\epsilon$, where $\forall{\emptyset\neq \S^p\subseteq \S^p\setminus\{p_k\}}, \S\subseteq N\setminus\{i\}, \mc_{p_k}(\S\cup \S^p) \leq \epsilon$. (2) is due to the definition of semivalues where the weights of coalitions sum to $1$. 
\end{proof}


\section{Algorithm~\ref{algo:fast_shapley}}\label{appendix:fast_shapley}

\input{sections/fast_shapley}

\section{Proofs for Theorem~\ref{thm:facility_location_shapley}}\label{appendix:facility_location}

Before deriving the Shapley and Banzhaf value for the facility location game, we first need to show the following mathematical identity which will be used for the derivation. 
\begin{lemma}
    \label{lem:binom}
    \begin{align}
      \sum_{k=0}^m \frac{\binom{m}{k}}{\binom{n}{k}} = \frac{n+1}{n+1-m}
    \end{align}
\end{lemma}
\begin{proof}[Proof Sketch.]
The identity can be shown in two steps: First, we show the identity $\frac{\tbinom{m}{k}}{\tbinom{n}{k}} = \frac{\tbinom{n-k}{m-k}}{\tbinom{n}{m}}$ by expansion of the terms. Then, we can take the denominator $\tbinom{n}{m}$ out of the summation over $k$, and as a common mathematical identity, the sum reduces to $\sum_{k=0}^m\tbinom{n-k}{m-k} = \tbinom{n+1}{m}$. Finally, by expanding the terms we arrive at $\frac{\tbinom{n+1}{m}}{\tbinom{n}{m}} = \frac{n+1}{n+1-m}.$
\end{proof}

\facilityshapley*

\begin{proof}
\textbf{(1) Proof for the Shapley value.}

    Let $v(\S) \coloneqq Fac(\S)$ and $n\coloneqq|\loc|$ as the number of players. Denote $w_{\S}$ as the weights of the Shapley value, i.e., $\vshapley_i = \sum_{\S\subseteq \loc\setminus\{i\}} w_{\S}\mc_i(\S)$, where $w_{\S} \coloneqq \frac{1}{n}\binom{n-1}{|\S|}^{-1}$. Observe that
    \begin{align}
      \varphi_i &= \sum_{\S \subseteq \loc \setminus\{i\}} w_{\S} \mc_i(\S)
         \stackrel{(*)}{=} \sum_{d\in D} \big[ \underbrace{\sum_{\S \subseteq \loc_{id}} w_{\S} u_{id}}_{(\#1)} - \underbrace{\sum_{\S \subseteq \loc_{id}} w_{\S} \max_{j \in \S} u_{jd}}_{(\#2)} \big],\nonumber
    \end{align}
    where $(*)$ is because the marginal contribution of $i$ for dimension $d$ is zero \emph{unless} $i$ is the largest element for that dimension and $\loc_{id}=\{j \in \loc \mid u_{jd} \leq u_{id}\}$ is the coalition of all elements which have smaller values in the $d$-th dimension than element $i$.
    
    Along each dimension $d$, $(\#1)$ is a weighted sum over coalitions $\S$ where $i$ is the largest element; and $(\#2)$ sums up for each $j\in \loc_{id}$ over all coalitions $\S\subseteq \loc_{id}$ where $j$ is the largest element. 
    We next compute $(\#1)$ and $(\#2)$ separately for each dimension $d\in D$:
    \begin{align*}
      (\#1) &=  \sum_{\S \subseteq \loc_{id}} w_{\S} u_{id} =  u_{id} \sum_{\S \subseteq \loc_{id}} w_{\S} \\
        &=  u_{id} \sum_{c=0}^{|\loc_{id}|} \sum_{\S \subseteq \loc_{id},|\S|=c} w_{\S} &&\textnormal{, where } w_{\S}\coloneqq \frac{1}{n}{\binom{n-1}{c}}^{-1}\\
        &= u_{id} \frac{1}{n} \sum_{c=0}^{|\loc_{id}|} {\binom{n-1}{c}}^{-1} \binom{|\loc_{id}|}{c} \\
        &\stackrel{(1)}{=}  u_{id} \frac{1}{n} \frac{n}{n-|\loc_{id}|} && \textnormal{, (1) by Lemma~\ref{lem:binom}}       \sum_{k=0}^m \frac{\binom{m}{k}}{\binom{n}{k}} = \frac{n+1}{n+1-m} \\
        &= u_{id} \frac{1}{n-|\loc_{id}|}
        , 
    \end{align*}
    
   Next we compute $(\#2)$. For simplicity, let $+,-$ denote set operations $\S\cup \{e\},\S\setminus \{e\}$, and denote $e^d_{it}$ is $t$-th largest element (after element $i$) in the $d$-th dimension.
    \begin{align*}
      (\#2) 
        &= \sum_{\S \subseteq \loc_{id}} w_{\S} \max_{j \in \S} u_{jd} \\
        &=  \sum_{\S \subseteq \loc_{id}-(e^d_{i1})} w_{\S+e^d_{i1}} u_{e^d_{i1}d} + \sum_{\S \subseteq \loc_{id}-(e^d_{i1}+e^d_{i2})} w_{\S+e^d_{i2}} u_{e^d_{i2}d} + \cdots \\
        &=  u_{e^d_{i1}d} \sum_{\S \subseteq \loc_{id}-(e^d_{i1})} w_{\S+e^d_{i1}} + u_{e^d_{i2}d} \sum_{\S\subseteq \loc_{id}-(e^d_{i1}+e^d_{i2})} w_{\S+e^d_{i2}} + \cdots  \\
        &=  u_{e^d_{i1}d} \underbrace{\sum_{\S \subseteq \loc_{id}-(e^d_{i1})} w_{\S+e^d_{i1}}}_{=: \beta_1} + u_{e^d_{i2}d} \underbrace{\sum_{\S \subseteq \loc_{id}-(e^d_{i1}+e^d_{i2})} w_{\S+e^d_{i2}}}_{=: \beta_2}  +\cdots
    \end{align*}
    \begin{align*}
      \textnormal{In particular, } \beta_t &= \sum_{\S \subseteq \loc_{id}-(e^d_{i1}+\ldots+e^d_{it})} w_{\S+e^d_{it}} \\
        &= \sum_{c=0}^{|\loc_{id}| - t }\sum_{\S \subseteq \loc_{id}-(e^d_{i1}+\ldots+e^d_{it}), |\S|=c} w_{\S+e^d_{it}} \\
        &= \frac{1}{n} \sum_{c=0}^{|\loc_{id}|-t } \binom{n-1}{c+1}^{-1} \binom{|\loc_{id}|-t}{c} \\
        &\stackrel{(1)}= \frac{1}{n} \sum_{c=0}^{|\loc_{id}|-t } \binom{n-1}{c+1}^{-1} \big[ \binom{|\loc_{id}|-t+1}{c+1} - \binom{|\loc_{id}|-t}{c+1} \big] \\
        &\stackrel{(2)}{=} \frac{1}{n} \sum_{x=1}^{|\loc_{id}|-t+1 } \binom{n-1}{x}^{-1} \big[ \binom{|\loc_{id}|-t+1}{x} - \binom{|\loc_{id}|-t}{x} \big] \\
        &\stackrel{(3)}{=}
         \frac{1}{n} \big[ \frac{n}{n-|\loc_{id}|+t-1} - 1 - \frac{n}{n-|\loc_{id}|+t} + 1 \big] \\
        &= \frac{1}{\lambda(t)+{\lambda(t)}^2} 
        ,
    \end{align*} 
    where $(1)$ is by Pascal's identity, $(2)$ by substituting $x=c+1$, $(3)$ by Lemma~\ref{lem:binom} and observing that $\binom{n}{k}$ is zero for $k>n$, and where $\lambda(t) = n-|\loc_{id}|+t-1$.
     \begin{align*}
      &\textnormal{Hence, }\val_i = \sum_{d\in D} \big[ u_{id} \frac{1}{n-|\loc_{id}|} - \sum_{t=1}^{|\loc_{id}|} \frac{1}{\lambda(t)+\lambda(t)^2} u_{e_{it}^d} \big].\qedhere
     \end{align*}

\textbf{(2) Proof for the Banzhaf value.}

  Similar to the proof for the Shapley value, we expand the Banzhaf value as follows:
    \begin{align*}
      \varphi(i) &= \sum_{\S \subseteq \loc - i} w_{\S} \mc_i(\S)
        = \sum_{d\in D} \big[ \underbrace{\sum_{\S \subseteq \loc_{id}} w_{\S} u_{id}}_{(\#1)} - \underbrace{\sum_{\S \subseteq \loc_{id}} w_{\S} \max_{j \in \S} u_{jd}}_{(\#2)} \big].
    \end{align*}
 By definition of the Banzhaf value $w_{\S} \coloneqq \frac{1}{2^{n-1}}$, next we compute $\#1$ and $\#2$.
    \begin{align*}
      (\#1) &= \sum_{d\in D} \sum_{\S \subseteq \loc_{id}} w_{\S} u_{id} \\
        &= \sum_{d\in D} u_{id} \sum_{\S \subseteq \loc_{id}} w_{\S} \\
        &= \sum_{d\in D} u_{id} \sum_{c=0}^{|\loc_{id}|} \sum_{\S \subseteq \loc_{id},|\S|=c} w_{\S} \\
        &= \sum_{d\in D} u_{id} \frac{1}{2^{n-1}}\sum_{c=0}^{|\loc_{id}|} \binom{|\loc_{id}|}{c} \\
        &= \frac{1}{2^{n-1}}\sum_{d\in D} 2^{|\loc_{id}|}u_{id}
    \end{align*}
    We then expand $(\#2)$ in a similar approach to the Shapley value (for notations c.f. above theorem),
    \begin{align*}
      (\#2) 
        = \sum_{\S \subseteq \loc_{id}} w_{\S} \max_{j \in \S} u_{jd} 
        =  u_{e^d_{i1}d} \underbrace{\sum_{\S \subseteq \loc_{id}-(e^d_{i1})} w_{\S+e^d_{i1}}}_{=: \beta_1} + u_{e^d_{i2}d} \underbrace{\sum_{\S \subseteq \loc_{id}-(e^d_{i1}+e^d_{i2})} w_{\S+e^d_{i2}}}_{=: \beta_2}  +\cdots
    \end{align*}
    \begin{align*}
      \textnormal{ where }\beta_t &= \sum_{\S \subseteq \loc_{id}-(e^d_{i1}+\ldots+e^d_{it})} w_{\S+e^d_{it}} \\
        &= \sum_{c=0}^{|\loc_{id}| - t }\sum_{\S \subseteq \loc_{id}-(e^d_{i1}+\ldots+e^d_{it}), |\S|=c} w_{\S+e^d_{it}} \\
        &= \frac{1}{2^{n-1}} \sum_{c=0}^{|\loc_{id}|-t } \binom{|\loc_{id}|-t}{c} \\
        &= \frac{1}{2^{n-1}} 2^{|\loc_{id}|-t}
    \end{align*} 
     \begin{align*}
      \textnormal{  Hence, } \val_i &= \frac{1}{2^{n-1}}\sum_{d\in D} \big[ 2^{|\loc_{id}|}u_{id} - \sum_{t=1}^{|\loc_{id}|} 2^{|\loc_{id}| -t} u_{e_{it}^d} \big]. \qedhere
     \end{align*}
  \end{proof}

%% file: sections/fast_shapley.tex
\begin{algorithm}[H]
	\caption{Efficient Shapley and Banzhaf value Computation for the Facility Location Game}
	\label{algo:fast_shapley}
	\begin{algorithmic}[1]	\Statex{\textbf{Input:} Locations $\loc$, customers $D$, utility matrix $U$}
	\Statex{\textbf{Output:} Shapley and Banzhaf value of all locations}
		\State{Sort the facility locations by (ascending) utility for each customer $d$, where  $U^{d\uparrow}$ is the sorted utility vector of customer $d$ and $\loc^{d\uparrow}$ are the sorted facility locations.}
		\For{each location $i \in \loc$}
		    \State{$l_{i}^d \leftarrow$ index of $i$ in $\loc^{d\uparrow}$, i.e.,  $l_{i}^d = |\loc_{id}| - 1$}
		    \State{$\vshapley_i \leftarrow \sum_{d\in D} [\frac{U_{id}}{n - l_i^d + 1} - \sum_{t=0}^{l_{id}}  \frac{U^{d\uparrow}_{l_i^d - t - 1}}{(n-l_i^d+t) + (n-l_i^d+t)^2}]$}
		    \State{$\vbanzhaf_i \leftarrow \frac{1}{2^{|\loc|-1}} \sum_{d\in D} [2^{|{l_i^d}| + 1} U_{id} - \sum_{t=0}^{l_i^d}  U^{d\uparrow}({l_i^d} - t + 1)]$}
		\EndFor
	\end{algorithmic} 
\end{algorithm}